%% file: main.tex
\documentclass[twoside]{article}

\usepackage[accepted]{aistats2020}

\usepackage{hyperref}

\include{macros}

\usepackage[round]{natbib}

\begin{document}

\runningtitle{Explaining the explainer}

\twocolumn[

\aistatstitle{Explaining the Explainer: A First Theoretical Analysis of LIME}

\aistatsauthor{Damien Garreau$^{1,3}$ \\ \texttt{damien.garreau@unice.fr} \And Ulrike von Luxburg$^{1,2}$ \\ \texttt{ulrike.luxburg@uni-tuebingen.de}}
\vspace{0.2cm}
\aistatsaddress{$^1$Max Planck Institute for Intelligent Systems, Germany \\ $^2$University of T\"ubingen, Germany \\ $^3$Universit\'e C\^ote d'Azur, Inria, CNRS, LJAD, France}]

\begin{abstract}
Machine learning is used more and more often for sensitive applications, sometimes replacing humans in critical decision-making processes. 
As such, interpretability of these algorithms is a pressing need. One popular algorithm to provide interpretability is LIME (Local Interpretable Model-Agnostic Explanation). 
In this paper, we provide the first theoretical analysis of LIME. We derive closed-form expressions for the coefficients of the interpretable model when the function to explain is linear. The good news is that these coefficients are proportional to the gradient of the function to explain: LIME indeed discovers meaningful features. However, our analysis also reveals that poor choices of parameters can lead LIME to miss important features.  
\end{abstract}

\section{Introduction}
\label{sec:introduction}

\subsection{Interpretability}
\label{sec:interpretability}

The recent advance of machine learning methods is partly due to the widespread use of very complicated models, for instance deep neural networks. 
As an example, the Inception Network \citep{Sze_Liu_Jia:2015} depends on approximately $23$ million parameters. 
While these models achieve and sometimes surpass human-level performance on certain tasks (image classification being one of the most famous), they are often perceived as \emph{black boxes}, with little understanding of how they make individual predictions. 

This lack of understanding is a problem for several reasons. 
First, it can be a source of catastrophic errors when these models are deployed \emph{in the wild}. 
For instance, for any safety system recognizing cars in images, we want to be absolutely certain that the algorithm is using features related to cars, and not exploiting some artifacts of the images. 
Second, this opacity prevents these models from being \emph{socially accepted}. 
It is important to get a basic understanding of the decision making process to accept it. 

Model-agnostic explanation techniques aim to solve this interpretability problem by providing qualitative or quantitative help to understand how black-box algorithms make decisions. 
Since the global complexity of the black-box models is hard to understand, they often rely on a \emph{local} point of view, and produce an interpretation for a specific instance. 
In this article, we focus on such an explanation technique: \textbf{Local Interpretable Model-Agnostic Explanations} (LIME, \citet{Rib_Sin_Gue:2016}). 

\begin{figure}[ht!]
	\vspace{-.1in}
	\begin{center}
		\includegraphics[scale=0.27]{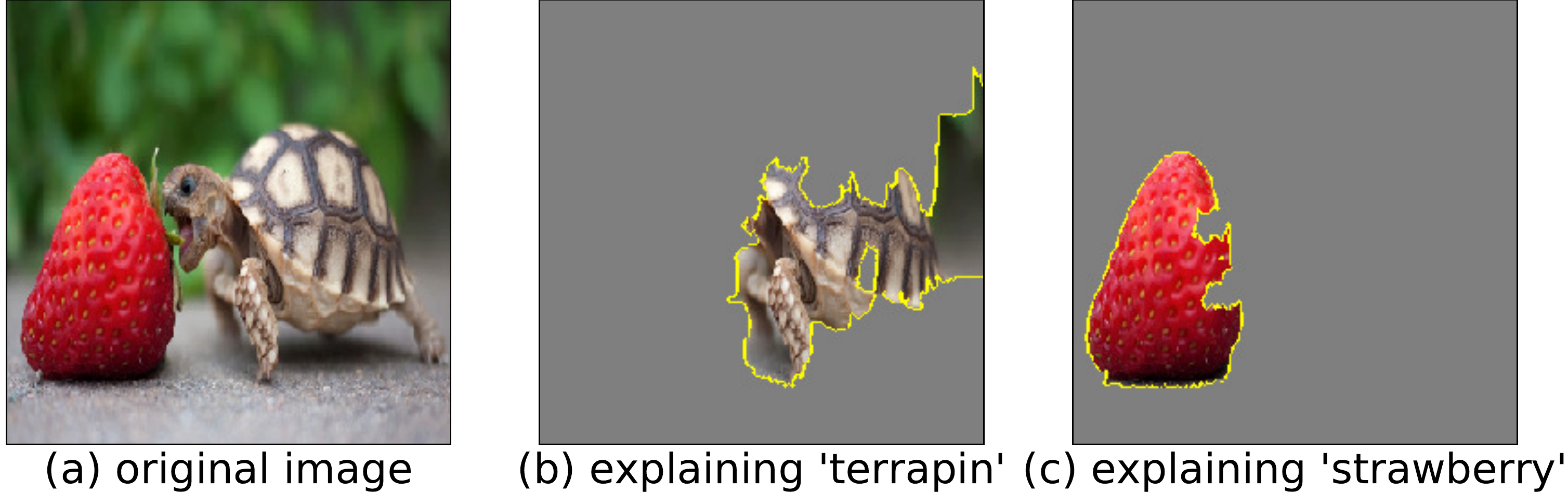}
	\end{center}
	\vspace{-.1in}
	\caption{\label{fig:lime-for-images}LIME explanation for object identification in images. We used Inception \citep{Sze_Liu_Jia:2015} as a black-box model. Terrapin, a sort of turtle, is the top label predicted for the image in panel~(a). 
	Panel~(b) shows the results of LIME, explaining how this prediction was made. 
	The highlighted parts of the image are the superpixels with the top coefficients in the surrogate linear model. % obtained by LIME.
	We ran the same experiment for the `strawberry' label in panel~(c). 
	}
\end{figure}

%%%%%%%%%%%%%%%%%%%%%%%%%%%%%%%%%%%%%%%%%%%%%%%%%%%%%%%%%%%%%%%%%%%%%%%%%%%%%%%%%%%%%%%%%%%%%%%%%%%%%%%%%%%%%%%%%%%%%%%%%%%%%

\subsection{Contributions}
\label{sec:contributions}

Our main goal in this paper is to provide theoretical guarantees for LIME. 
On the way, we shed light on some interesting behavior of the algorithm in a simple setting. 
Our analysis is based on the Euclidean version of LIME, called ``tabular LIME.'' 
Our main results are the following:
\begin{enumerate}[noitemsep,topsep=0pt]
	\item When the model to explain is linear, we \textbf{compute in closed-form} the average coefficients of the surrogate linear model obtained by \tlime. % , up to random error terms. 
	\item In particular, these coefficients are \textbf{proportional to the partial derivatives of the black-box model} at the instance to explain. This implies that \tlime indeed highlights important features. 
	\item On the negative side, using the closed-form expressions we show that \textbf{it is possible to make some important features disappear} in the interpretation, just by changing a parameter of the method. 
	\item We also compute the local error of the surrogate model, and show that it is \textbf{bounded} away from~$0$ in general. 
\end{enumerate}

We explain how \tlime works in more details in Section~\ref{sec:lime-outline}. 
In Section~\ref{sec:main-results}, we state our main results. 
They are discussed in Section~\ref{sec:discussion}, and we provide an outline of the proof of our main result in Section~\ref{sec:proofs}. 
We conclude in Section~\ref{sec:conclusion}. 

%%%%%%%%%%%%%%%%%%%%%%%%%%%%%%%%%%%%%%%%%%%%%%%%%%%%%%%%%%%%%%%%%%%%%%%%%%%%%%%%%%%%%%%%%%%%%%%%%%%%%%%%%%%%%%%%%%%%%%%%%%%%%%%%%

\section{LIME: Outline and notation}
\label{sec:lime-outline}

\subsection{Intuition}
\label{sec:intuition}

From now on, we will consider a particular model encoded as a function $f:\Reals^\Dim \to\Reals$ and a particular instance $\xi\in\Reals^\Dim$ to explain. 
We make no assumptions on this function, \emph{e.g.}, how it might have been learned. 
We simply consider~$f$ as a black-box model giving us predictions for all points of the input space. 
Our goal will be to explain the decision $f(\xi)$ that this model makes for one particular instance~$\xi$. 

As soon as~$f$ is too complicated, it is hopeless to try and fit an interpretable model globally, since the interpretable model will be too simple to capture all the complexity of~$f$. 
Thus a reasonable course of action is to consider a \emph{local} point of view, and to explain~$f$ in the neighborhood of some fixed instance~$\xi$. 
This is the main idea behind LIME: To explain a decision for some fixed input $\xi$, sample other examples around~$\xi$, use these samples to build a simple interpretable model in the neighborhood of~$\xi$, and use this surrogate model to explain the decision for~$\xi$.  

One additional idea that makes a huge difference with other existing methods is to use \emph{discretized} features of smaller dimension $\Dim'$ to build the local model. 
These new categorical features are easier to interpret, since they are categorical. 
In the case of images, they are built by using a split of the image~$\xi$ into superpixels \citep{Ren_Mal:2003}. 
See Figure~\ref{fig:lime-for-images} for an example of LIME output in the case of image classification. 
In this situation, the surrogate model highlights the superpixels of the image that are the most ``active'' in predicting a given label. 

Whereas LIME is most famous for its results on images, it is easier to understand how it operates and to analyze theoretically on \textbf{tabular data}. 
In the case of tabular data, LIME works essentially in the same way, with a main difference: tabular LIME requires a train set, and each feature is discretized according to the empirical quantiles of this training set. 

\begin{figure}[ht]
	\vspace{-0.1in}
	\begin{center}
	\includegraphics[scale=0.25]{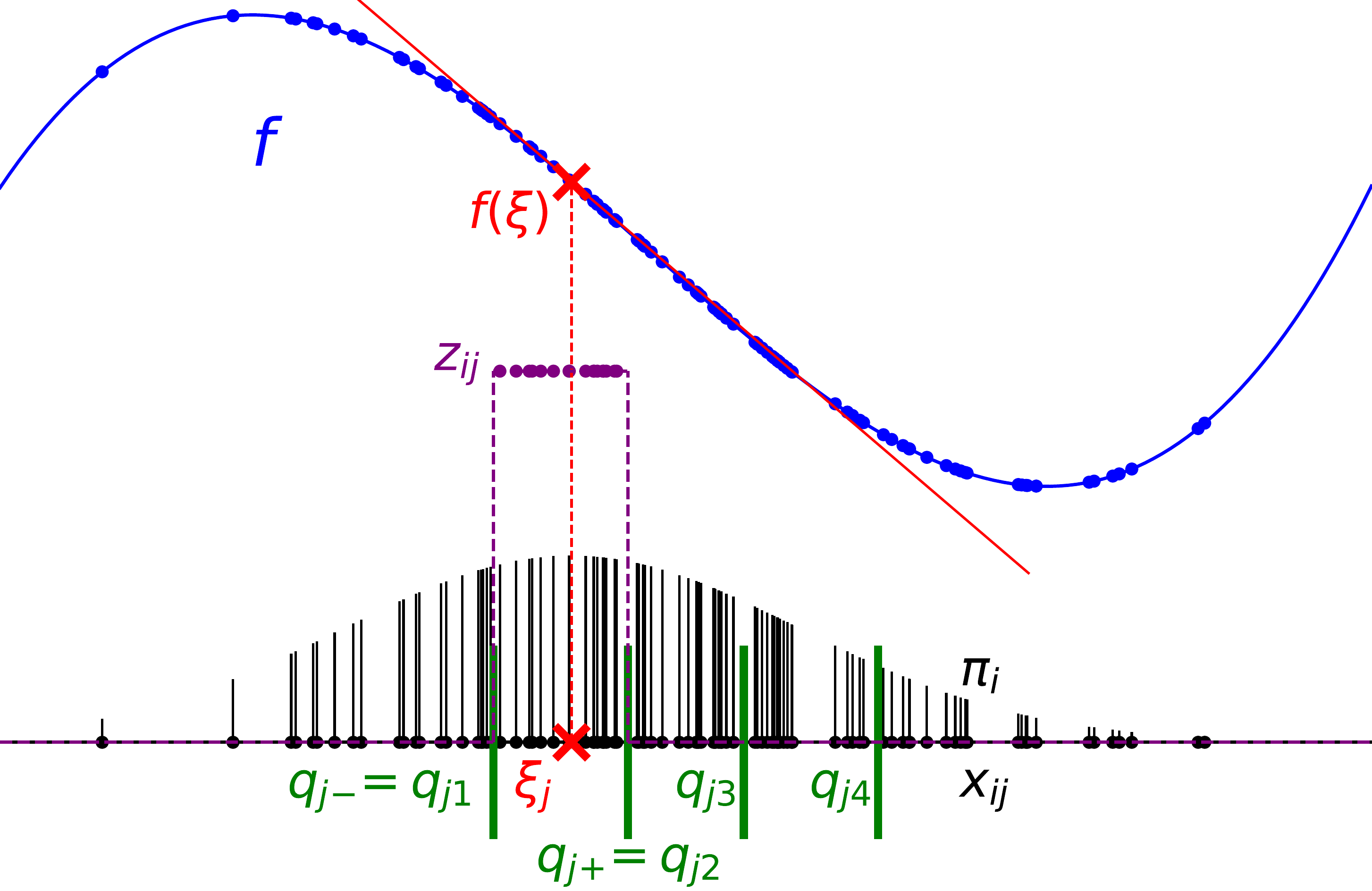}
	\end{center}
	\vspace{-0.2in}
	\caption{\label{fig:general_setting}General setting of \tlime along coordinate~$j$. Given a specific datapoint $\xi$ (in red), we want to build a local model for~$f$ (in blue), given new samples $x_1,\ldots,x_n$ (in black). 
	Discretizing with respect to the quantiles of the distribution (in green), these new samples are transformed into categorical features~$z_i$ (in purple). 
	In the construction of the surrogate model, they are weighted with respect to their proximity with~$\xi$ (here exponential weights given by Eq.~\eqref{eq:def-weights}, in black). In red, we plotted the tangent line, the best linear approximation one could hope for. 
	}
\end{figure}

We now describe the general operation of LIME on Euclidean data, which we call \tlime. 
We provide synthetic description of \tlime in Algorithm~\ref{algo:tlime}, and we refer to Figure~\ref{fig:general_setting} for a depiction of our setting along a given coordinate. 
Suppose that we want to explain the prediction of the model~$f$ at the instance~$\xi$. %, given a train set. 
\tlime has an intricate way to sample points in a local neighborhood of~$\xi$.   
First, \tlime constructs empirical quantiles of the train set on each dimension, for a given number $p$ of bins. 
These quantile boxes are then used to construct a discrete representation of the data: if $\xi_j$ falls between $\empquantile_k$ and $\empquantile_{k+1}$, it receives the value~$k$. 
We now have a discrete version of~$\xi$, say $(2,3,\dots)^\top$. 
The next step is to sample discrete examples in $\{1,\ldots,p\}^\Dim$ uniformly at random: for instance, $(1,3,\ldots)^\top$ means that \tlime sampled an encoding such that the first coordinate falls into the first quantile box, the second coordinate into the third, etc. 
\tlime subsequently un-discretizes these encodings by sampling from a normal distribution truncated to the corresponding quantile boxes, obtaining \emph{new examples} $x_1,\ldots,x_n$. 
For example, for sample $(1,3,\ldots)^\top$ we now sample the first coordinate from a normal distribution restricted to quantile box~$\#1$, the second coordinate from quantile box~$\#3$, etc. 
This sampling procedure ensures that we have samples in each part of the space. 
The next step is to convert these sampled points to binary features, indicating for each coordinate if the new example falls into the same quantile box as~$\xi$. 
Here, $z_i$ would be $(1,0,\ldots)^\top$. 
Finally, an interpretable model (say linear) is learned using these binary features. 

\begin{algorithm}[ht]
	\caption{\tlime for regression}
	\label{algo:tlime}
	\begin{algorithmic}[1]% [x] means every x line numbered
		\REQUIRE Model $f$, $\#$ of new samples $n$, instance~$\xi$, bandwidth $\nu$, $\#$ of bins~$p$, mean $\mu$, variance $\sigma^2$
		\STATE $q\leftarrow$ \texttt{GetQuantiles}($p$,$\mu$,$\sigma$)
		\STATE $t \leftarrow$ \texttt{Discretize}($\xi$,$q$)
		\FOR{$i=1$ to $n$}
		\FOR{$j=1$ to $\Dim$}
		\STATE $y_{i,j}\leftarrow$ \texttt{SampleUniform}($\{1,\ldots,p\}$)
		\STATE $(q_\ell,q_u)\leftarrow (q_{j,y_{ij}},q_{j,y_{ij}+1})$ 
		\STATE $x_{i,j}\leftarrow$ \texttt{SampleTruncGaussian}($q_\ell,q_u,\mu,\sigma$)
		\STATE $z_{i,j}\leftarrow \indic{t_j=y_{i,j}}$
		\ENDFOR
		\STATE $\pi_i\leftarrow \exp\left(\frac{-\norm{x_i-\xi}^2}{2\nu^2}\right)$
		\ENDFOR
		\STATE $\betahat\leftarrow$\texttt{WeightedLeastSquares}($z,f(x),\pi$)
		\RETURN $\betahat$
	\end{algorithmic}
\end{algorithm}

\subsection{Implementation choices and notation}
\label{sec:tabular-lime}

LIME is a quite general framework and leaves some freedom to the user regarding each brick of the algorithm. 
We now discuss each step of \tlime in more detail, presenting our implementation choices and introducing our notation on the way. 

\medskip

\textbf{Discretization. }
As said previously, the first step of \tlime is to create a partition of the input space using a train set. 
Intuitively, \tlime produces \emph{interpretable features} by discretizing each dimension. 
Formally, given a fixed number of bins~$p$, for each feature~$j$, the empirical quantiles $\empquantile_{j,0},\ldots,\empquantile_{j,p}$ are computed. 
Thus, along each dimension, there is a mapping $\Empdisc_j:\Reals \to \{1,\ldots,p\}$ associating each real number to the index of the quantile box it belongs to. 
For any point $x\in\Reals^\Dim$, the interpretable features are then defined as a $0-1$ vector corresponding to the discretization of~$x$ being the same as the discretization of~$\xi$. 
Namely, $z_{j}=\indic{\Empdisc_j(x)=\Empdisc_j(\xi)}$ for all $1\leq j\leq \Dim$. 
Intuitively, these categorical features correspond to the \emph{absence} or \emph{presence} of interpretable components. 
The discretization process makes a huge difference with respect to other methods: we lose the obvious link with the gradient of the function, and it is much more complicated to see how the local properties of~$f$ influence the result of the LIME algorithm, even in a simple setting. %(in the case of linear regression at least). 
In all our experiments, we took $p=4$ (quartile discretization, the default setting).

\begin{figure}[h]
\vspace{-0.1in}
\begin{center}
\includegraphics[scale=0.30]{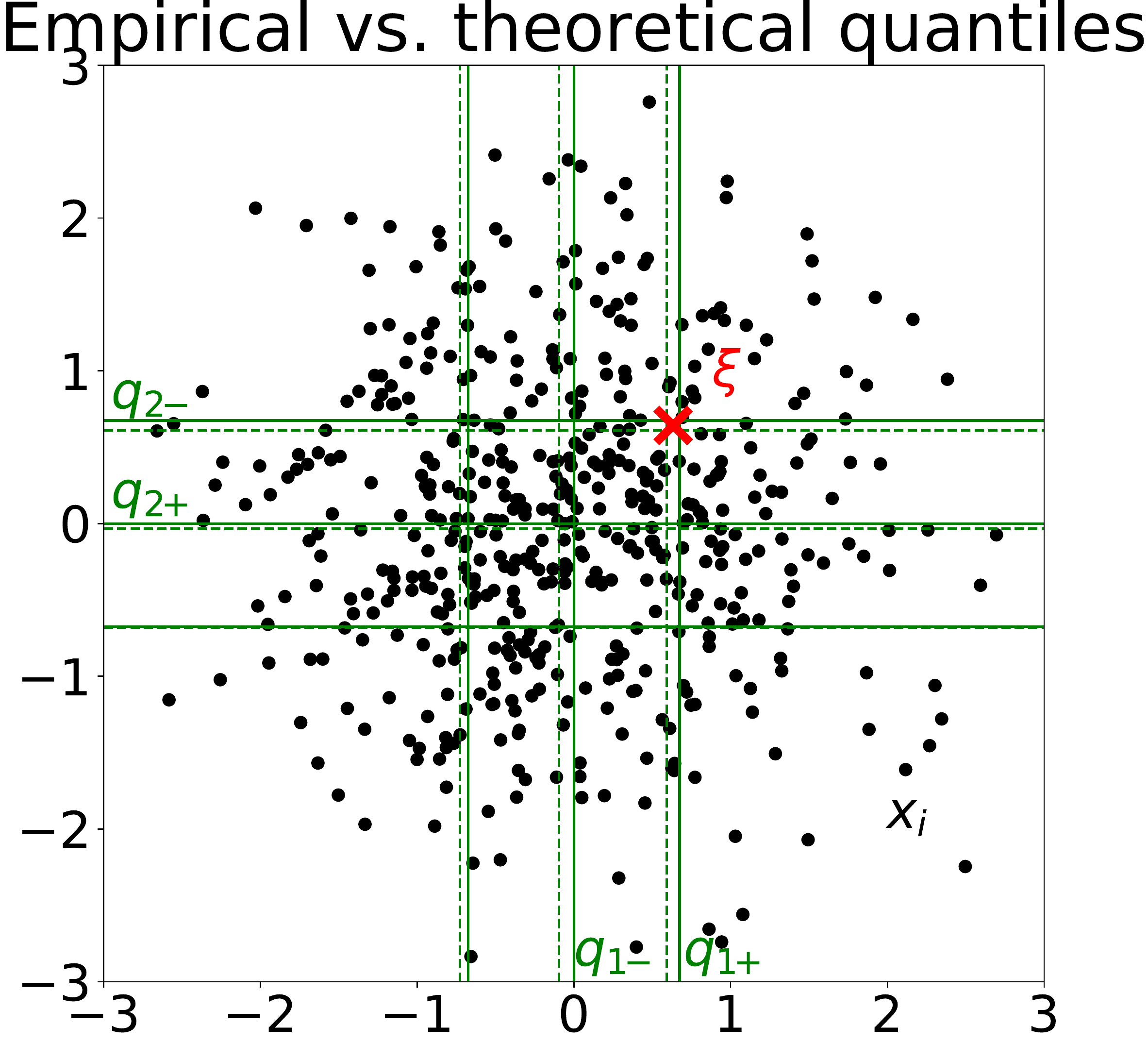}
\end{center}
\vspace{-0.1in}
\caption{\label{fig:quantiles_investigation}A visualization of the train set in dimension $\Dim=2$ with $\mu=(0,0)^\top$, and $\sigma^2=1$. The empirical quantiles (dashed green lines) are already very close to the theoretical quantiles (green lines) for $n_{\text{train}}=500$. 
The main difference in the procedure appears if~$\xi$ (red cross) is chosen at the edge of a quantile box, changing the way all the new samples are encoded. 
But for a train set containing enough observations and a generic~$\xi$, there is virtually no difference between using the theoretical quantiles and the empirical quantiles. 
}
%NOTE: We make sure this does not happens in our experiments. }
\end{figure}

\textbf{Sampling strategy. }
Along with~$\Empdisc$, \tlime creates an un-discretization procedure $\Empundisc:\{1,\ldots,p\}\to\Reals$. 
Simply put, given a coordinate~$j$ and a bin index~$k$, $\Empundisc_j(k)$ samples a truncated Gaussian on the corresponding bin, with parameters computed from the training set. 
The \tlime sampling strategy for a new example amounts to (i) sample $y_i\in \{1,\ldots,p\}^d$ a random variable such that the $y_{ij}$ are independent samples of the discrete uniform distribution on $\{1,\ldots,p\}$, and (ii) apply the un-discretization step, that is, return $\Empundisc(y)$. 
We will denote by $x_1,\ldots,x_n\in\Reals^\Dim$ these new examples, and $z_1,\ldots,z_n\in\{0,1\}^\Dim$ their discretized counterparts. 
Note that it is possible to take other bin boxes than those given by the empirical quantiles, the $y_{ij}$s are then sampled according to the frequency observed in the dataset. 
The sampling step of \tlime helps to explore the values of the function in the neighborhood of the instance to explain. 
Thus it is not so important to sample according to the distribution of the data, and a Gaussian sampling that mimics it is enough. 

Assuming that we know the distribution of the train data, it is possible to use the theoretical quantiles instead of the empirical ones. 
For a large number of examples, they are arbitrary close (see, for instance, Lemma~21.2 in \citet{Van:2000}).
See Figure~\ref{fig:quantiles_investigation} for an illustration. 
It is this approach that we will take from now on: we denote the discretization step by~$\phi$ and denote the quantiles by $q_{jk}$ for $1\leq j\leq \Dim$ and $0\leq k\leq p$ to mark this slight difference. 
Also note that, for every $1\leq j\leq \Dim$, we set $q_{j\pm}$ the quantiles bounding $\xi_j$, that is, $q_{j-}\leq \xi_j < q_{j+}$ (see Figure~\ref{fig:general_setting}).

\paragraph{Train set. }
\tlime requires a train set, which is left free to the user. 
In spirit, one should sample according to the distribution of the train set used to fit the model~$f$. 
Nevertheless, this train set is rarely available, and from now on, we choose to consider draws from a $\gaussian{\mu}{\sigma^2\Id_{\Dim}}$. 
The parameters of this Gaussian can be estimated from the training data that was used for~$f$ if available. 
Thus, in our setting, along each dimension $j$, the $\left(q_{jk}\right)_{0\leq k\leq p}$ are the (rescaled) quantiles of the normal distribution. 
In particular, they are identical for all features. 
A fundamental consequence is that sampling the new examples $x_i$s first and then discretizing \textbf{has the same distribution} as sampling first the bin indices $y_{i}$s and then un-discretizing.   

\paragraph{Weights. }
We choose to give each example the weight
\begin{equation}
\label{eq:def-weights}
\pi_i \defeq \exp\left(\frac{-\norm{x_i-\xi}^2}{2\nu^2}\right)
\, ,
\end{equation}
where $\norm{\cdot}$ is the Euclidean norm on $\Reals^\Dim$ and $\nu > 0$ is a bandwidth parameter. 
It should be clear that~$\nu$ is a hard parameter to tune:
\begin{itemize}[noitemsep,topsep=0pt]
\item if~$\nu$ is very large, then \textbf{all the examples receive positive weights:}
we are trying to build a simple model that captures the complexity of~$f$ at a global scale. This cannot work if~$f$ is too complicated. 
\item if~$\nu$ is too small, then \textbf{only examples in the immediate neighborhood of~$\xi$} receive positive weights. Given the discretization step, this amounts to choosing $z_i=(1,\ldots,1)^\top$ for all $i$. Thus the linear model built on top would just be a constant fit, missing all the relevant information. 
\end{itemize}

Note that other distances than the Euclidean distance can be used, for instance the cosine distance for text data. 
The default implementation of LIME uses $\norm{z_i-t}$ instead of $\norm{x_i-\xi}$, with bandwidth set to $0.75\Dim$. 
We choose to use the true Euclidean distance between~$\xi$ and the new examples as it can be seen as a smoothed version of the distance to~$z_i$ and has the same behavior. 

\paragraph{Interpretable model. }

The final step in \tlime is to build a local interpretable model. 
Given a class of simple, interpretable models~$G$, \tlime selects the best of these models by solving 
\begin{equation}
\label{eq:tabular-lime-general}
\argmin_{g\in G} \biggl\{L_n(f,g,\pi_{\xi}) + \reg{g}\biggr\}
\, ,
\end{equation}
where $L_n$ is a local loss function evaluated on the new examples $x_1,\ldots,x_n$, and $\Omega:\Reals^\Dim\to\Reals$ is a regularizer function. 
For instance, a natural choice for the local loss function is the weighted squared loss 
\begin{equation}
\label{eq:local-loss}
L_n(f,g,\pi) \defeq \frac{1}{n}\sum_{i=1}^n \pi_i \left(f(x_i)-g(z_i)\right)^2
\, .
\end{equation}
We saw in Section~\ref{sec:interpretability} different possibilities for~$G$. 
In this paper, we will focus exclusively on the linear models, in our opinion the easiest models to interpret. 
Namely, we set $g(z_i) = \beta^\top z_i + \beta_0$, with $\beta\in\Reals^\Dim$ and $\beta_0\in\Reals$. 
To get rid of the intercept~$\beta_0$, we now use the standard approach to introduce a phantom coordinate~$0$, and $z,\beta\in\Reals^{\Dim+1}$ with $z_0 = 1$ and $\beta_0 = \beta_0$. 
We also stack the $z_i$s together to obtain $Z\in\{0,1\}^{n\times (\Dim +1)}$. 

The regularization term $\Omega(g)$ is added to insure further interpretability of the model by reducing the number of non-zero coefficients in the linear model given by \tlime.
Typically, one uses $L^2$ regularization (ridge regression is the default setting of LIME) or $L^1$ regularization (the Lasso). 
To simplify the analysis, we will set $\Omega = 0$ in the following. 
We believe that many of the results of Section~\ref{sec:main-results} stay true in a regularized setting, especially the switch-off phenomenon that we are going to describe below: coefficients are even more likely to be set to zero when $\Omega\neq 0$. 

In other words, in our case \tlime performs \emph{weighted linear regression} on the interpretable features $z_i$s, and outputs a vector $\betahat\in\Reals^{\Dim+1}$ such that  
\begin{equation}
\label{eq:tabular-lime-main}
\betahat \in \argmin_{\beta\in\Reals^{\Dim+1}} \left\{\frac{1}{n} \sum_{i=1}^n \pi_i (y_i - \beta^\top z_i)^2 \right\}
\, .
\end{equation}
Note that $\betahat$ is a random quantity, with randomness coming from the sampling of the new examples $x_1,\ldots,x_n$. 
It is clear that from a theoretical point of view, a big hurdle for the theoretical analysis is the discretization process (going from the~$x_i$s to the $z_i$s). 

\paragraph{Regression vs. classification. }
To conclude, let us note that \tlime can be used both for regression and classification. 
Here we focus on the \emph{regression} mode: the outputs of the model are real numbers, and not discrete elements. 
In some sense, this is a more general setting than the classification case, since the classification mode operates as \tlime for regression, but with $f$ chosen as the function that gives the likelihood of belonging to a certain class according to the model. 

%%%%%%%%%%%%%%%%%%%%%%%%%%%%%%%%%%%%%%%%%%%%%%%%%%%%%%%%%%%%%%%%%%%%%%%%%%%%%%%%%%%%%%%%%%%%%%%%%%%%%%%%%%%%%%%%%%%%%%%%%%%%%%%%

\subsection{Related work}
\label{sec:related-work}

%In this work, we focus on \emph{model-agnostic} methods for interpretability, that is, methods which do not require the underlying model to belong to a particular class.  
%Before describing \tlime, the center of our attention, 
Let us mention a few other model-agnostic methods that share some characteristics with LIME. 
We refer to \citet{Gui_Mon_Rug:2019} for a thorough review. 

\paragraph{Shapley values. }
Following \citet{Sha:1953} the idea is to estimate for each subset of features~$S$ the expected prediction difference $\Delta(S)$ when the value of these features are \emph{fixed} to those of the example to explain. 
The contribution of the $j$th feature is then set to an average of the contribution of~$j$ over all possible coalitions (subgroups of features not containing~$j$).
They are used in some recent interpretability work, see \citet{Lun_Lee:2017} for instance. 
It is extremely costly to compute, and does not provide much information as soon as the number of features is high. 
Shapley values share with LIME the idea of quantifying how much a feature contributes to the prediction for a given example. 

\paragraph{Gradient methods. }
Also related to LIME, \emph{gradient-based} methods as in \citet{Bae_Sch_Har:2010} provide local explanations without knowledge of the model. 
Essentially, these methods compute the partial derivatives of~$f$ at a given example. 
For images, this can yield satisfying plots where, for instance, the contours of the object appear: a \emph{saliency map} \citep{Zei_Fer:2014}. 
\citet{Shr_Gre_Shc:2016,Shr_Gre_Kun:2017} propose to use the ``input $\times$ derivative'' product, showing advantages over gradient methods. 
But in any case, the output of these gradient based methods is not so interpretable since the number of features is so high. 
LIME gets around this problem by using a local dictionary with much smaller dimensionality than the input space. 

%%%%%%%%%%%%%%%%%%%%%%%%%%%%%%%%%%%%%%%%%%%%%%%%%%%%%%%%%%%%%%%%%%%%%%%%%%%%%%%%%%%%%%%%%%%%%%%%%%%%%%%%%%%%%%%%%%%%%%%%%%%%%

\section{Theoretical value of the coefficients of the surrogate model}
\label{sec:main-results}

We are now ready to state our main result. 
Let us denote by~$\betahat$ the coefficients of the linear surrogate model obtained by \tlime. 
In a nutshell, when the underlying model~$f$ is linear, we can derive the average value~$\beta$ of the $\betahat$ coefficients. 
In particular, we will see that the $\beta_j$s are proportional to the partial derivatives $\partial_j f(\xi)$. 
The exact form of the proportionality coefficients is given in the formal statement below, it essentially depends on the scaling parameters
\[
\mutilde \defeq \frac{\nu^2\mu + \sigma^2\xi}{\nu^2+\sigma^2}\in\Reals^\Dim\,\text{and}\,\, \sigmatilde \defeq \frac{\nu^2\sigma^2}{\nu^2+\sigma^2} >0
\, ,
\]
and the $q_{j\pm}$s, the quantiles left and right of the $\xi_j$s. 

\begin{theorem}[Coefficients of the surrogate model, theoretical values]
\label{th:main-result-light}
Assume that~$f$ is of the form $x\mapsto a^\top x + b$, and set 
\begin{equation}
\label{eq:beta}
\beta \defeq  \begin{pmatrix}
f(\mutilde) + \sum_{j=1}^\Dim \frac{a_j\theta_j}{1-\alpha_j} \\
\frac{-a_1\theta_1}{\alpha_1(1-\alpha_1)} \\
\vdots \\
\frac{-a_\Dim \theta_\Dim}{\alpha_\Dim(1-\alpha_\Dim)}
\end{pmatrix}
\in\Reals^{\Dim+1}
\, ,
\end{equation}
where, for any $1\leq j\leq \Dim$, we defined
\[
\alpha_j \defeq \left[\frac{1}{2}\erfun{\frac{x-\mutilde_j}{\sigmatilde\sqrt{2}}}\right]_{q_{j-}}^{q_{j+}}
\, ,
\]
and
\[
\theta_j \defeq \left[ \frac{\sigmatilde}{\sqrt{2\pi}}\exp\left(\frac{-(x-\mutilde_j)^2}{2\sigmatilde^2}\right)\right]_{q_{j-}}^{q_{j+}}
\, .
\]
Let $\eta \in (0,1)$. 
Then, with high probability greater than $1-\eta$, it holds that
\begin{equation*}
%\label{eq:main-result}
\norm{\betahat-\beta} \lesssim \max(\sigma\norm{\nabla f},f(\mutilde) + \sigmatilde\norm{\nabla f})\sqrt{\frac{\log 1/\eta}{n}}
\, .
\end{equation*}
\end{theorem}

A precise statement with the accurate dependencies in the dimension and the constants hidden in the result can be found in the Appendix (Theorem~\ref{th:main-result}).
Before discussing the consequences of Theorem~\ref{th:main-result-light} in the next section, remark that since $\xi$ is encoded by $(1,1,\ldots,1)^\top$, the prediction of the local model at $\xi$, $\fhat(\xi)$, is just the sum of the $\betahat_j$s. 
According to Theorem~\ref{th:main-result-light}, $\fhat(\xi)$ will be close to this value, with high probability. 
Thus we also have a statement about the error made by the surrogate model in~$\xi$.

\begin{corollary}[Local error of the surrogate model]
	\label{cor:local-model-error-light}
	Let $\eta \in (0,1)$. 
	Then, under the assumptions of Theorem~\ref{th:main-result-light}, with probability greater than $1-\eta$, it holds that
	\begin{align*}
	&\abs{\fhat(\xi) - f(\mutilde) +\sum_{j=1}^\Dim \frac{a_j\theta_j}{\alpha_j}} \leq \\
	&\phantom{blabla}\leq  \max(\sigma\norm{\nabla f},f(\mutilde) + \sigmatilde\norm{\nabla f})\sqrt{\frac{\log 1/\eta}{n}}
	\, ,
	\end{align*}
	with hidden constants depending on $\Dim$ and the $\alpha_j$s. 
\end{corollary}

Obviously the goal of \tlime is not to produce a very accurate model, but to provide interpretability. 
The error of the local model can be seen as a hint about how reliable the interpretation might be. 

\begin{figure}[ht!]
	\vspace{-0.1in}
	\begin{center}
		\includegraphics[scale=0.25]{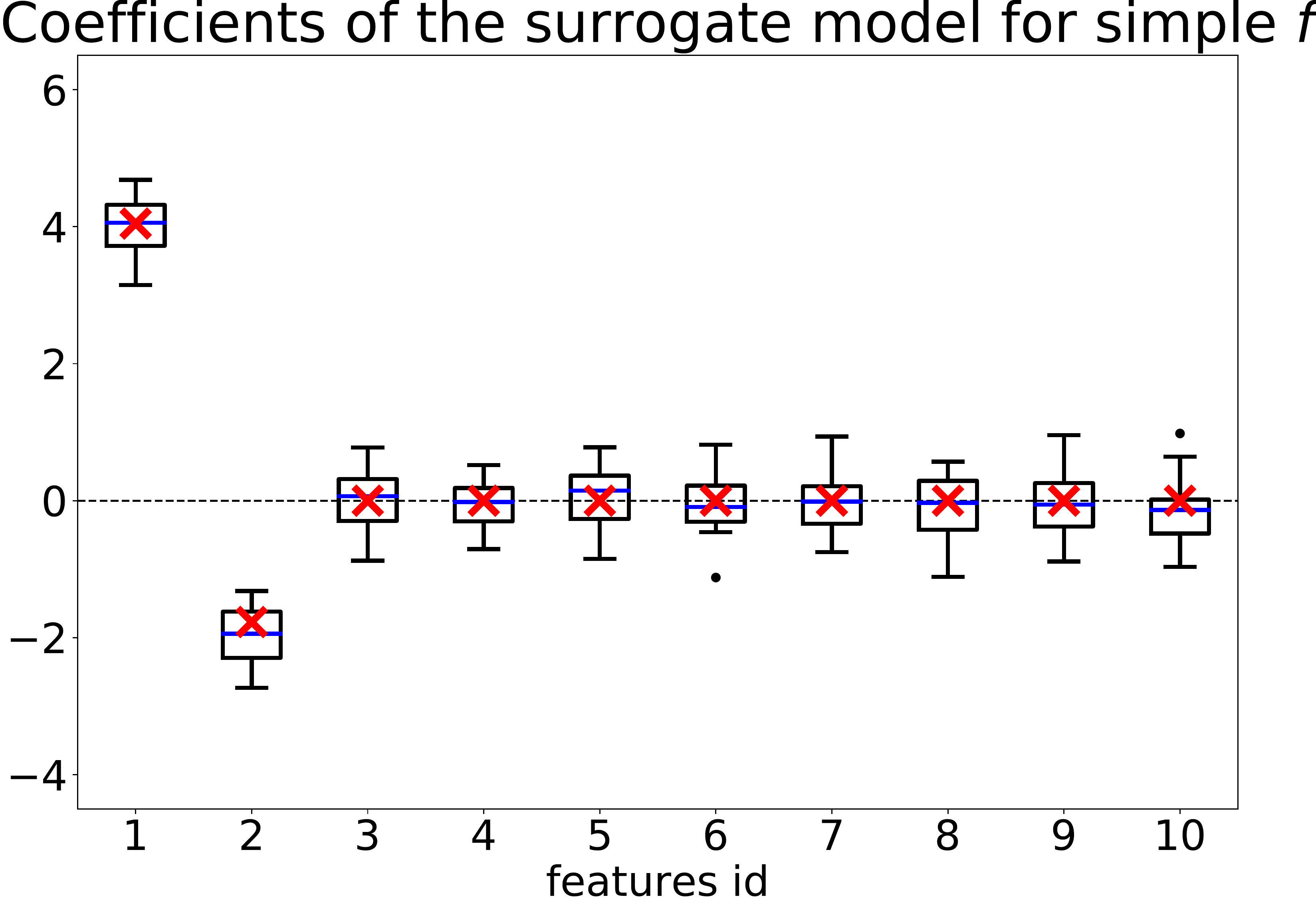}
	\end{center}
	\vspace{-0.2in}
	\caption{\label{fig:simple_f}
	Example where the true underlying black box model only depends on \emph{two} features: $f(x)=10x_1-10x_2$. For each of the $10$ features, we plot the values of the $\betahat_j$s obtained by \tlime. 
	The blue line shows the median over all experiments, the red cross the $\beta_j$ theoretical value according to our theorem. The boxplots contain values between first and third quartiles, the whiskers are $1.5$ times the interquartile ranges, and the black dots mark values outside this range. 
	To produce the figure, we made $20$ repetitions of the experiment, with $n=10^4$ examples and $\nu=1$.
	We see that \tlime finds nonzero coefficients exactly for the first two coordinates, up to noise coming from the sampling. 
	This is the result that one would hope to achieve, and also the result predicted by our theory. 
	}
\end{figure}

\section{Consequences of our main results}
\label{sec:discussion}

We now discuss the consequences of Theorem~\ref{th:main-result-light} and Corollary~\ref{cor:local-model-error-light}. 

\paragraph{Dependency in the partial derivatives. }
A first consequence of Theorem~\ref{th:main-result-light} is that the coefficients of the linear model given by \tlime are approximately \textbf{proportional to the partial derivatives of~$f$} at~$\xi$, with constant depending on our assumptions. 
An interesting follow-up is that, if~$f$ depends only on a few features, then the partial derivatives in the other coordinates are zero, and the coefficients given by \tlime for these coordinates will be~$0$ as well. 
For instance, if $f(x)=10x_1-10x_2$ as in Figure~\ref{fig:simple_f}, then $\beta_1\simeq 11.4$, $\beta_2\simeq -4.1$, and $\beta_j=0$ for all $j\geq 3$. 
In a simple setting, we thus showed that \tlime does not produce interpretations with additional erroneous feature dependencies. 
Indeed, when the number of samples is high, the coordinates which do not influence the prediction will have a coefficient close to the theoretical value~$0$ in the surrogate linear model. 
%We illustrate this phenomenon in Figure~\ref{fig:simple_f}. 
For a bandwidth not too large, this dependency in the partial derivatives seems to hold to some extent for more general functions. 
See for instance Figure~\ref{fig:second-order-perturbations}, where we demonstrate this phenomenon for a kernel regressor.

\begin{figure}[ht!]
	\vspace{-0.1in}
	\begin{center}
		\includegraphics[scale=0.25]{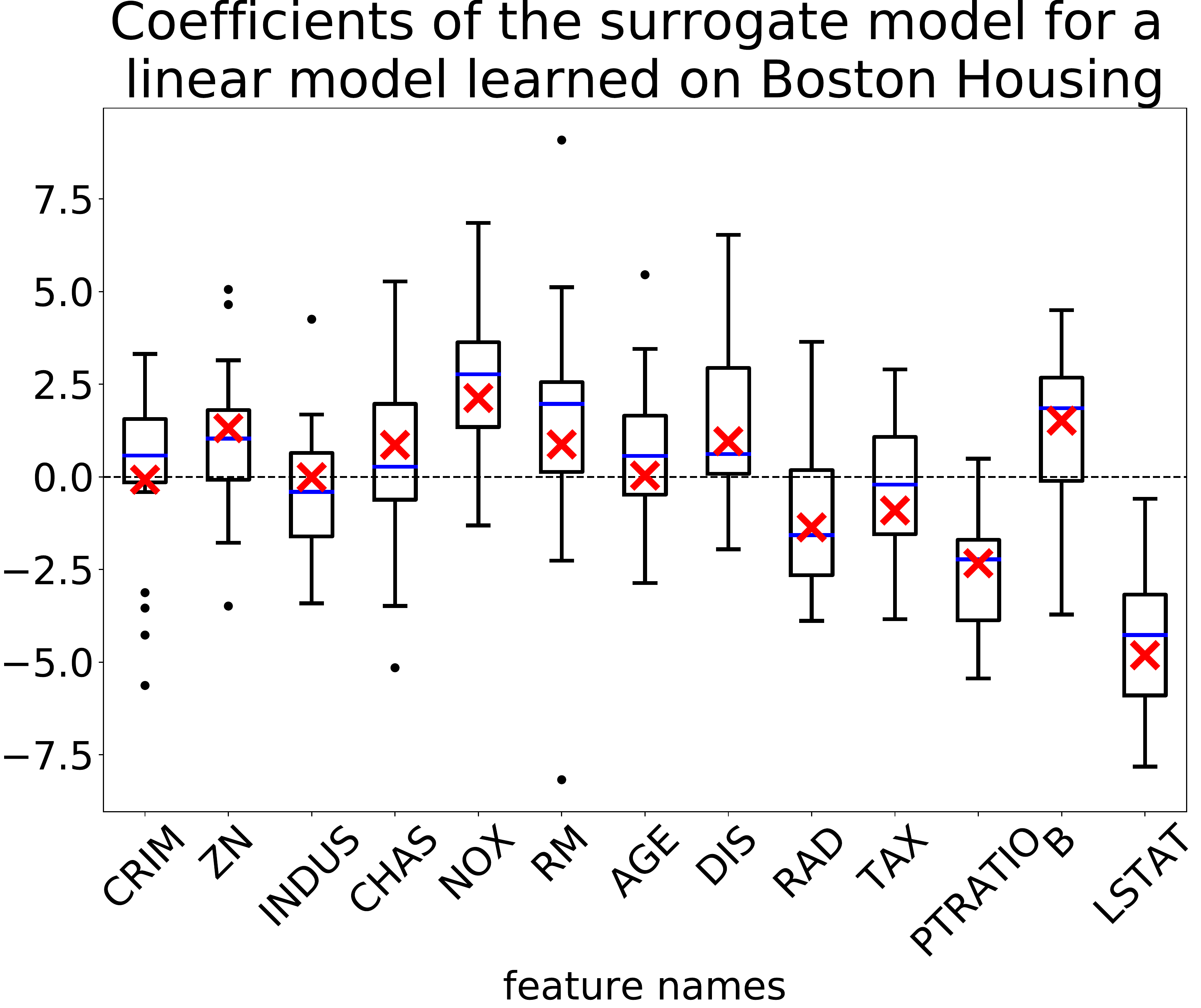}
	\end{center}
	\vspace{-0.2in}
	\caption{\label{fig:theoretical_explanation_vs_empirical}Values of the coefficients obtained by \tlime on each coordinate in dimension $\Dim=13$ for a linear model trained on the Boston housing dataset \citep{Har_Rub:1978}. 
	The $\beta_j$s are concentrated around the red crosses, which denote the $\beta_j$s, the theoretical values predicted by Theorem~\ref{th:main-result-light}. 
	To produce the figure, we ran $20$ experiments with $n=10^3$ new samples generated for each run and we set $\nu=1$.
	}
\end{figure}

\medskip

\textbf{Robustness of the explanations. } 
Theorem~\ref{th:main-result-light} means that, for large~$n$, \tlime outputs coefficients that are very close to~$\beta$ with high probability, where~$\beta$ is a vector that can be computed explicitly as per Eq.~\eqref{eq:beta}. 
Still without looking too closely at the values of~$\beta$, this is already interesting and hints that there is some robustness in the interpretations provided by \tlime: given enough samples, the explanation will not jump from one feature to the other. 
This is a desirable property for any interpretable method, since the user does not want explanations to change randomly with different runs of the algorithm. 
We illustrate this phenomenon in Figure~\ref{fig:theoretical_explanation_vs_empirical}.

\begin{figure}[ht]
	\begin{center}
		\includegraphics[scale=0.25]{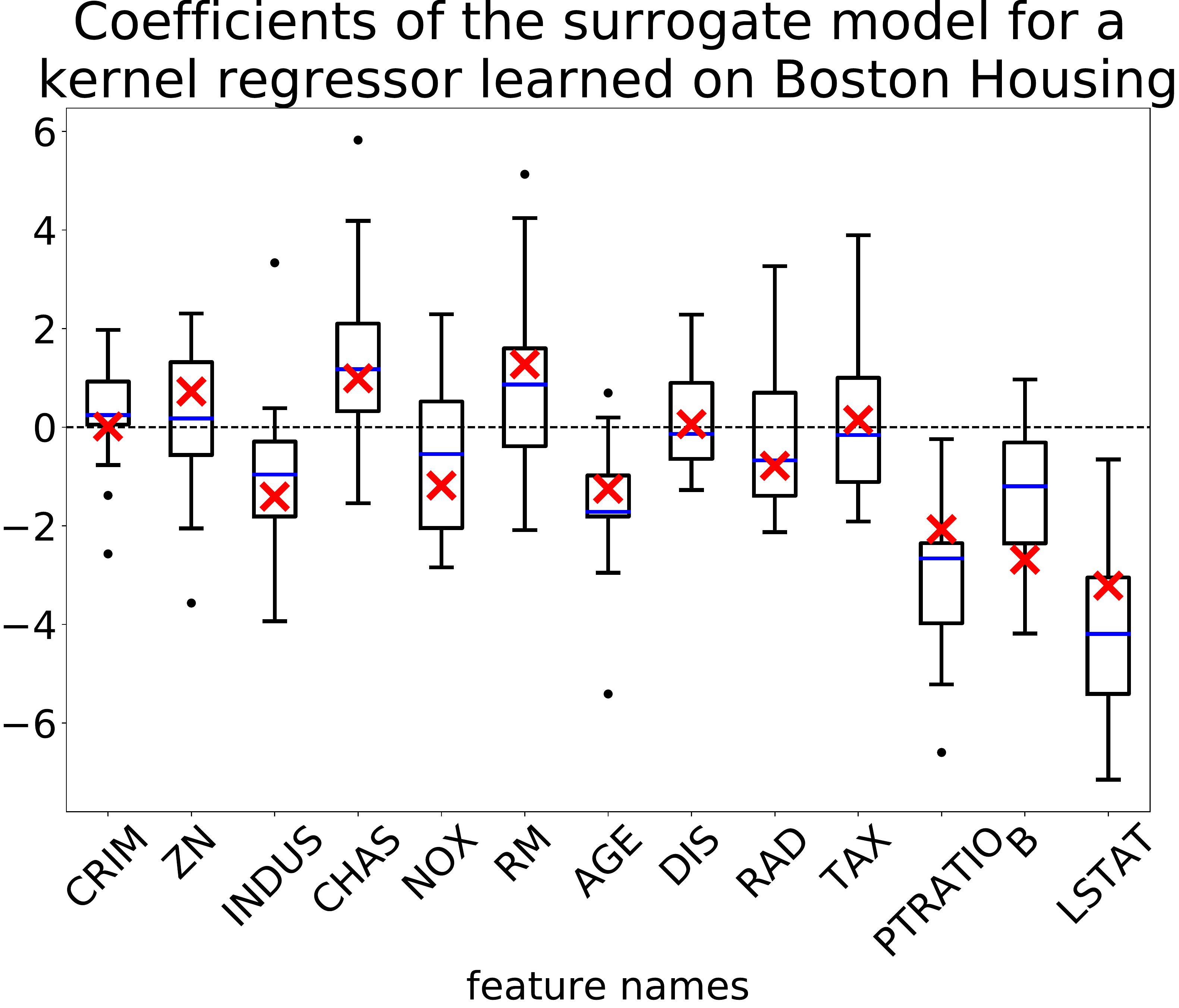}
	\end{center}
	\vspace{-0.2in}
	\caption{\label{fig:second-order-perturbations}Values of the coefficients obtained by \tlime on each coordinate. We used the same settings as in Figure~\ref{fig:theoretical_explanation_vs_empirical}, but this time we train a \emph{kernel ridge} regressor on the Boston Housing dataset---a nonlinear function. For the ridge regression, we used the Gaussian kernel with scale parameter set to~$5$ and default regularization constant ($\alpha=1$). 
	We then estimated the partial derivatives of~$f$ at~$\xi$ and reported the corresponding $\beta_j$s in red. 
	For the chosen bandwidth (we took $\nu=1$), the experiments seem to roughly agree with our theory.  
	}
\end{figure}

\paragraph{Influence of the bandwidth. }
Unfortunately, Theorem~\ref{th:main-result-light} does not provide directly a founded way to pick~$\nu$, which would for instance minimize the variance for a given level of noise. 
The quest for a founded heuristic is still open. 
However, we gain some interesting insights on the role of~$\nu$. 
Namely, for fixed $\xi$, $\mu$, and $\sigma$, the multiplicative constants $\theta_j/(\alpha_j(1-\alpha_j))$ appearing in Eq.~\eqref{eq:beta} depend essentially on~$\nu$. 

Without looking too much into these constants, one can already see that they regulate the magnitude of the coefficients of the surrogate model in a non-trivial way. 
For instance, in the experiment depicted in Figure~\ref{fig:simple_f}, the partial derivative of~$f$ along the two first coordinate has the same magnitude, whereas the interpretable coefficient is much larger for the first coordinate than the second. 
Thus we believe that the value of the coefficients in the obtained linear model should not be taken too much into account. 

More disturbing, it is possible to artificially (or by accident) put $\theta_j$ to zero, therefore \textbf{forgetting} about feature~$j$ in the explanation, whereas it could play an important role in the prediction. 
To see why, we have to return to the definition of the $\theta_j$s: 
since $q_{j-}<q_{j+}$ by construction, to have $\theta_j=0$ is possible only if 
\begin{equation}
\label{eq:special-nu}
\nucrit\defeq \sigma^2 \frac{2\xi_j - q_{j-}-q_{j+}}{-2\mu_j+q_{j-}+q_{j+}} > 0
\, ,
\end{equation}
and $\nu^2$ is set to~$\nucrit$. 
We demonstrate this switching-off phenomenon in Figure~\ref{fig:switch-off}. 
%We see this as a severe limitation of the method. 
An interesting take is that~$\nu$ not only decides at which scale the explanation is made, but also the magnitude of the coefficients in the interpretable model, even for small changes of~$\nu$. 

\begin{figure}[ht!]
\begin{center}
\includegraphics[scale=0.25]{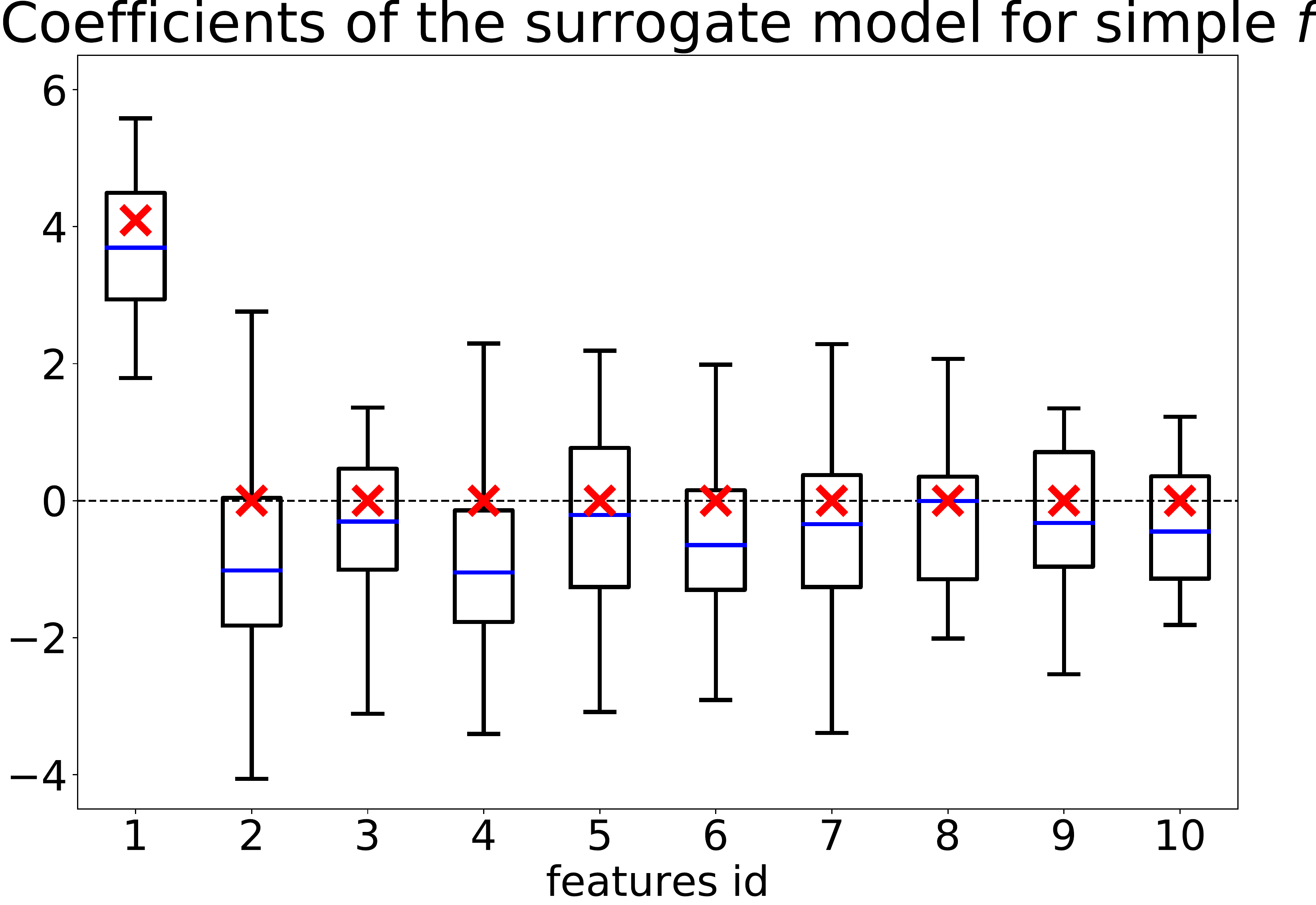}
\end{center}
\vspace{-0.2in}
\caption{\label{fig:switch-off}Values of the coefficients given by LIME. In this experiment, we took exactly the same setting as in Figure~\ref{fig:simple_f}, but this time set the bandwidth to $\nu=0.53$ instead of $1$. 
In that case, the second feature is switched-off by \tlime. 
Note that it is not the case that~$\nu$ is too small and that we are in a degenerated case: \tlime still puts a nonzero coefficient on the first coordinate. 
}
\end{figure}

\paragraph{Error of the surrogate model. }
A simple consequence of Corollary~\ref{cor:local-model-error-light} is that, unless some  cancellation happens between in the term $f(\mutilde)-\sum_j \frac{a_j\theta_j}{\alpha_j}$, \textbf{the local error of the surrogate model is bounded away from zero}. 
For instance, as soon as $\mutilde\neq \mu$, it is the general situation. 
Therefore, the surrogate model produced by \tlime is not \emph{accurate} in general. 
We show some experimental results in Figure~\ref{fig:error_local_model}. 

Finally, we discuss briefly the limitations of Theorem~\ref{th:main-result-light}. 

\paragraph{Linearity of $f$. }
The linearity of~$f$ is a quite restrictive assumption, but we think that it is useful to consider for two reasons. 

First, the weighted nature of the procedure means that \tlime is not considering examples that are too far away from~$\xi$ with respect to the scaling parameter~$\nu$. 
Thus it is truly a \emph{local} assumption on~$f$, that could be replaced by a boundedness assumption on the Hessian of~$f$ in the neighborhood of~$\xi$, at the price of more technicalities and assuming that~$\nu$ is not too large. 
See, in particular, Lemma~\ref{lemma:gaussian-second-order} in the Appendix, after which we discuss an extension of the proof when~$f$ is linear with a second degree perturbative term.
We show in Figure~\ref{fig:second-order-perturbations} how our theoretical predictions behave for a non-linear function (a kernel ridge regressor).

Second, our main concern is to know whether \tlime operates correctly in a simple setting, and not to provide bounds for the most general~$f$ possible. 
Indeed, if we can already show imperfect behavior for \tlime when~$f$ is linear as seen earlier, our guess is that such behavior will only worsen for more complicated~$f$. 

\begin{figure}[ht!]
	\vspace{-0.1in}
	\begin{center}
	\includegraphics[scale=0.25]{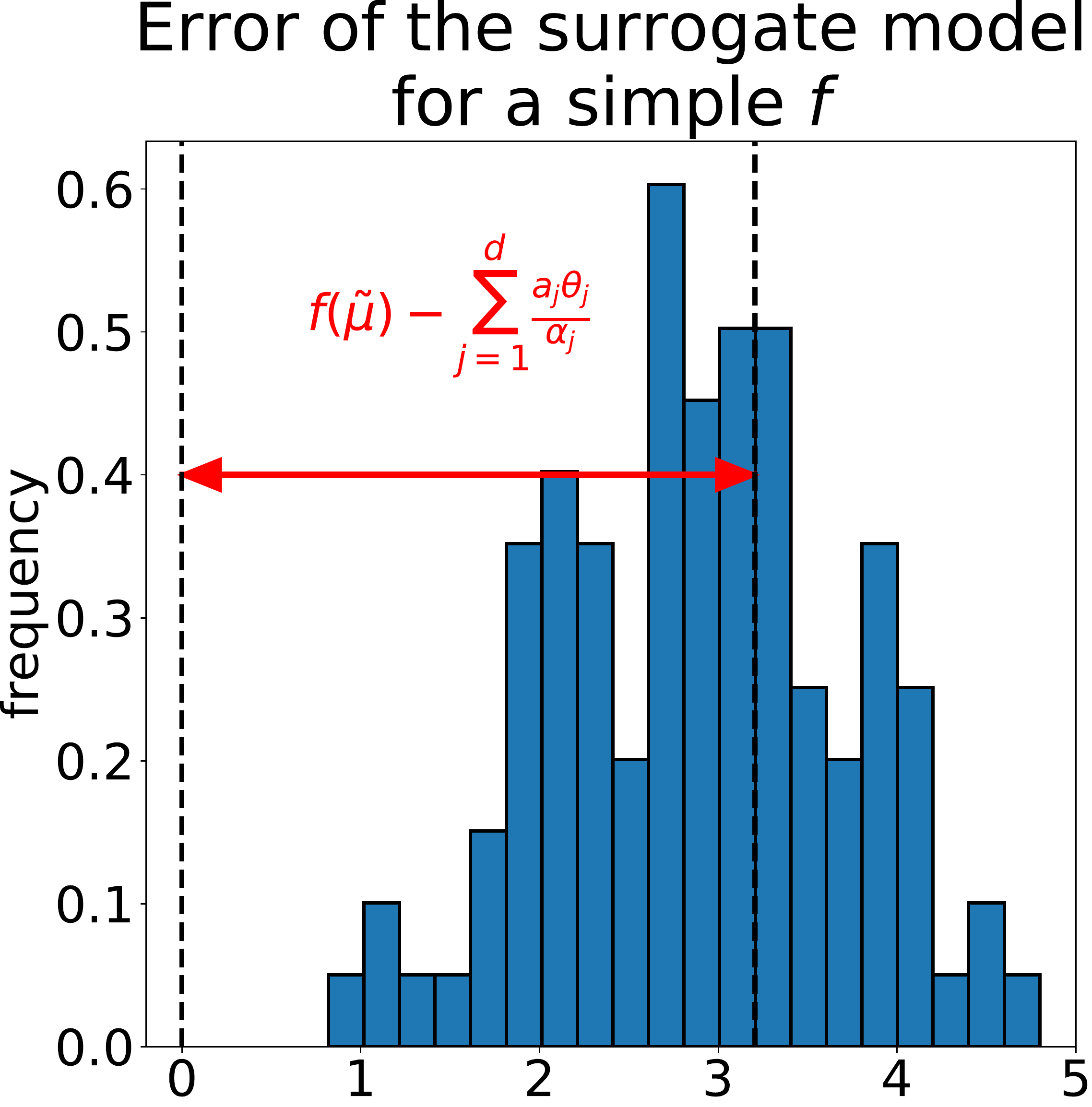}
	\end{center}
	\vspace{-0.2in}
	\caption{\label{fig:error_local_model}Histogram of the errors $\fhat(\xi)-f(\xi)$. 
	The setting is the same as in Figure~\ref{fig:simple_f}, but we repeated the experiment $100$ times. 
	The red double arrow marks the value given by Corollary~\ref{cor:local-model-error-light} around which the local error concentrate. 
	With high probability, the error of the surrogate model is bounded away from~$0$. 
	}
\end{figure}

\paragraph{Sampling strategy. }

In our derivation, we use the theoretical quantiles of the Gaussian distribution along each axis, and not prescribed quantiles. 
We believe that the proof could eventually be adapted, but that the result would loose in clarity. 
Indeed, the computations for a truncated Gaussian distribution are far more convoluted than for a Gaussian distribution. 
For instance, in the proof of Lemma~\ref{lemma:expected-covariance-matrix} in the Appendix, some complicated quantities depending on the prescribed quantiles would appear when computing $\expec{\pi_i z_{ik}}$. 

%%%%%%%%%%%%%%%%%%%%%%%%%%%%%%%%%%%%%%%%%%%%%%%%%%%%%%%%%%%%%%%%%%%%%%%%%%%%%%%%%%%%%%%%%%%%%%%%%%%%%%%%%%%%%%%%%%%%%%%%

\section{Proof of Theorem~\ref{th:main-result-light}}
\label{sec:proofs}

In this section, we explain how Theorem~\ref{th:main-result-light} is obtained. 
All formal statements and proofs are in the Appendix. 

\paragraph{Outline. }
%
% main idea
The main idea underlying the proof is to realize that $\betahat$ is the solution of a weighted least squares problem. 
Denote by $\Pi\in\Reals^{n\times n}$ the diagonal matrix such that $\Pi_{ii}=\pi_i$ (the \emph{weight matrix}), and set $f(x)\in\Reals^{\Dim+1}$ the response vector. 
Then, taking the gradient of Eq.~\eqref{eq:weighted-least-squares-main}, one obtains the key equation 
\begin{equation}
\label{eq:weighted-least-squares-main}
(Z^\top \Pi Z)\betahat = Z^\top \Pi f(x)
\, .
\end{equation}
Let us define $\Sigmahat \defeq \frac{1}{n}Z^\top \Pi Z$ and $\Gammahat \defeq \frac{1}{n}Z^\top \Pi f(x)$, as well as their population counterparts $\Sigma\defeq\Expec[\Sigmahat]$ and $\Gamma\defeq\Expec[\Gammahat]$. 
Intuitively, if we can show that $\Sigmahat$ and $\Gammahat$ are close to $\Sigma$ and $\Gamma$, assuming that $\Sigma$ is invertible, then we can show that $\betahat$ is close to $\beta \defeq \Sigmainv\Gamma$. 

% technical difficulties
The main difficulties in the proof come from the \textbf{non-linear} nature of the new features~$z_i$, introducing tractable but challenging integrals. 
Fortunately, the Gaussian sampling of LIME allows us to overcome these challenges (at the price of heavy computations). 

\paragraph{Covariance matrix. }
The first part of our analysis is thus concerned with the study of the empirical covariance matrix $\Sigmahat$. 
Perhaps surprisingly, it is possible to compute the population version of $\Sigmahat$:
\[
\Sigma = \cst\begin{pmatrix}
1 & \alpha_1 & \cdots & \alpha_\Dim \\
\alpha_1 & \alpha_1 & & \alpha_i\alpha_j \\
\vdots & & \ddots & \\
\alpha_\Dim & \alpha_i\alpha_j & & \alpha_\Dim 
\end{pmatrix}
\, ,
\]
where the $\alpha_j$s were defined in Section~\ref{sec:main-results}, and~$\cst$ is a scaling constant that does not appear in the final result (see Lemma~\ref{lemma:expected-covariance-matrix}). % in the Appendix). 

Since the $\alpha_j$s are always distinct from $0$ and $1$, the special structure of~$\Sigma$ makes it possible to invert it in closed-form.  
We show in Lemma~\ref{lemma:inverse-covariance-matrix} that
\[
\cst^{-1}\!\begin{pmatrix}
1+\sum_{j=1}^\Dim \frac{\alpha_j}{1-\alpha_j} & \frac{-1}{1-\alpha_1} & \cdots & \frac{-1}{1-\alpha_\Dim} \\
\frac{-1}{1-\alpha_1} & \frac{1}{\alpha_1(1-\alpha_1)} & & 0 \\
\vdots & & \ddots & \\
\frac{-1}{1-\alpha_\Dim} & 0 & & \frac{1}{\alpha_\Dim(1-\alpha_\Dim)}
\end{pmatrix}
\, .
\]
We then achieve control of $\opnorm{\Sigmahatinv - \Sigmainv}$ \emph{via} standard concentration inequalities, since the new samples are Gaussian and the binary features are \emph{bounded} (see Proposition~\ref{prop:control-operator-norm-inverse-covariance}). 

\paragraph{Right-hand side of Eq.~\eqref{eq:weighted-least-squares-main}. }
Again, despite the non-linear nature of the new features, it is possible to compute the expected version of $\Gammahat$ in our setting. 
%under our assumptions on~$f$. 
%
In this case, we show in Lemma~\ref{lemma:computation-expected-response-vector} that % in the Appendix):
\[
\Gamma = \cst \begin{pmatrix}
f(\mutilde) \\
\alpha_1 f(\mutilde) - a_1 \theta_1 \\
\vdots \\
\alpha_\Dim f(\mutilde) - a_\Dim \theta_\Dim 
\end{pmatrix}
\, ,
\]
where the $\theta_j$s were defined in Section~\ref{sec:main-results}. 
They play an analogous role to the $\alpha_j$s but, as noted before, they are signed quantities. 
As with the analysis of the covariance matrix, since the weights %are $[0,1]$-valued 
and the new features are %binary, 
bounded, it is possible to show a concentration result for~$\Gammahat$ (see Lemma~\ref{lemma:concentration-response}). 

\paragraph{Concluding the proof. }
We can now conclude, first upper bounding $\norm{\betahat - \Sigmainv \Gamma}$ by
\[
\opnorm{\Sigmahatinv}\norm{\Gammahat - \Gamma} + \opnorm{\Sigmahatinv-\Sigmainv}\norm{\Gamma}
\, ,
\]
and then controlling each of these terms using the previous concentration results. 
The expression of~$\beta$ is simply obtained by multiplying $\Sigmainv$ and $\Gamma$. 

\section{Conclusion and future directions}
\label{sec:conclusion}

In this paper we provide the first theoretical analysis of LIME, with some good news (LIME discovers interesting features) and bad news (LIME might forget some important features and the surrogate model is not faithful). 
All our theoretical results are verified by simulations. 

For future work, we would like to complement these results in various directions: 
Our main goal is to extend the current proof to any function by replacing~$f$ by its Taylor expansion at~$\xi$. 
On a more technical side, we would like to extend our proof to other distance functions (\emph{e.g.}, distances between the $z_i$s and $\xi$, which is the default setting of LIME), to non-isotropic sampling of the $x_i$s (that is, $\sigma$ not constant across the dimensions), and to ridge regression. 

\subsubsection*{Acknowledgements}

The authors would like to thank Christophe Biernacki for getting them interested in the topic, as well as Leena Chennuru Vankadara for her careful proofreading. 
This work has been supported by the German Research Foundation through the Institutional Strategy of the University of T\"ubingen (DFG, ZUK 63), the Cluster of Excellence ``Machine Learning---New Perspectives for Science'' (EXC 2064/1 number 390727645), and the BMBF Tuebingen AI Center (FKZ: 01IS18039A).

%\newpage
\bibliography{biblio}
\bibliographystyle{abbrvnat}

\include{appendix}

\end{document}

%% file: macros.tex
\usepackage{lmodern}
\usepackage[english]{babel}
\usepackage[T1]{fontenc} 
\usepackage[latin1]{inputenc} 
\usepackage[babel=true]{csquotes}
\usepackage{xspace}

\usepackage{breakcites}% for citations that go over a page

\usepackage{enumitem}% vertical spacing in itmize
\usepackage[htt]{hyphenat}% cutting \texttt
\usepackage{microtype}% everything smaller, uncomment last

\usepackage{amsmath}
\usepackage{amssymb}
\usepackage{amsthm}
\usepackage{mathrsfs}
\usepackage{mathabx}
\usepackage{dsfont}
\usepackage{mathtools}

\usepackage{algorithm}
\usepackage{algorithmic}
\DeclareMathOperator*{\argmin}{arg\,min} % argmin
\DeclareMathOperator{\Gaussian}{\mathcal{N}} % gaussian distribution
\DeclareMathOperator{\Proba}{\mathbb{P}} % probability of an event
\DeclareMathOperator{\Reals}{\mathbb{R}} % real numbers
\newcommand{\abs}[1]{\left\lvert#1\right\rvert} % absolute value
\newcommand{\defeq}{\vcentcolon =} % define equals
\newcommand{\Diag}{\mathrm{Diag}}% diagonal operator -> matrix
\newcommand{\diff}{\mathop{}\mathopen{}\mathrm{d}} % straight d for \int
\newcommand{\Dim}{d}% dimension
\newcommand{\exps}[1]{\mathrm{e}^{#1}}% small exponential of something
\newcommand{\Expec}{\mathbb{E}}% symbol for expected value
\newcommand{\expec}[1]{\Expec\left[#1\right]} % expected value of something
\newcommand{\Erfun}{\mathrm{erf}}% error function
\newcommand{\erfun}[1]{\Erfun\left(#1\right)}% error function
\newcommand{\Id}{\mathrm{I}}% identity matrix
\newcommand{\indic}[1]{\mathbf{1}_{#1}} % indicator function
\newcommand{\frobnorm}[1]{\norm{#1}_{\mathrm{F}}}% Frobenius norm of a matrix
\newcommand{\gaussian}[2]{\Gaussian\left(#1,#2\right)}% Gaussian distribution
\newcommand{\norm}[1]{\left\lVert#1\right\rVert}% norm
\renewcommand{\epsilon}{\varepsilon}% I don't like the look of \epsilon
\DeclareRobustCommand{\proba}[1]{\ensuremath{\Proba\left (#1\right )}} % proba

\numberwithin{equation}{section}

\theoremstyle{definition}

\newtheorem{remark}{Remark}
\numberwithin{remark}{section}

\newtheorem{hypo}{H\!\!}
\theoremstyle{plain}
\newtheorem{proposition}{Proposition}
\numberwithin{proposition}{section}
\newtheorem{lemma}{Lemma}
\numberwithin{lemma}{section}
\newtheorem{theorem}{Theorem}
\numberwithin{theorem}{section}
\newtheorem{corollary}{Corollary}
\numberwithin{corollary}{section}

\newcommand{\alphacst}{A_{\Dim}}% constant depending on the alphas
\newcommand{\betahat}{\widehat{\beta}}% estimator for \beta
\newcommand{\cst}{C_\Dim}% the main constant
\newcommand{\empquantile}{\hat{q}}% empirical quantiles
\newcommand{\Empdisc}{\hat{\phi}}% empirical discretization function
\newcommand{\Empundisc}{\hat{\psi}}% undiscretization
\newcommand{\fhat}{\hat{f}}% estimate
\newcommand{\Gammahat}{\widehat{\Gamma}}% empirical Gamma
\newcommand{\mutilde}{\tilde{\mu}}
\newcommand{\nucrit}{V_{\text{crit}}}% critical squared bandwidth
\newcommand{\reg}[1]{\Omega\left(#1\right)}% a regularizer
\newcommand{\opnorm}[1]{\norm{#1}_{\text{op}}}% operator norm of a matrix
\newcommand{\Sigmahat}{\widehat{\Sigma}}% estimated covariance matrix
\newcommand{\Sigmahatinv}{\Sigmahat^{-1}}% inverse of the estimated covariance matrix
\newcommand{\Sigmainv}{\Sigma^{-1}}% inverse of the covariance matrix
\newcommand{\sigmatilde}{\tilde{\sigma}}
\newcommand{\tlime}{\texttt{TabularLIME}\xspace}% shortcut for tabular lime

\makeatletter
\def\th@plain{%
	\thm@notefont{}% same as heading font
	\itshape % body font
}
\def\th@definition{%
	\thm@notefont{}% same as heading font
	\normalfont % body font
}
\makeatother

%% file: appendix.tex
\onecolumn

\begin{center}
\mbox{}\\
{\Large Supplementary material for:}\\
\vspace{0.5cm}
{\Huge Explaining the Explainer: A First Theoretical Analysis of LIME}
\mbox{}\\\mbox{}
\end{center}

% general organisation
In this supplementary material, we provide the proof of Theorem~\ref{th:main-result-light} of the main paper. 
It is a simplified version of Theorem~\ref{th:main-result}. 
We first recall our setting in Section~\ref{sec:appendix:setting}. 
Then, following Section~\ref{sec:proofs} of the main paper, we study the covariance matrix in Section~\ref{sec:appendix:covariance}, and the right-hand side of the key equation~\eqref{eq:weighted-least-squares-main} in Section~\ref{sec:appendix:response-vector}. 
Finally, we state and prove Theorem~\ref{th:main-result} in Section~\ref{sec:appendix:conclusion}. 
Some technical results (mainly Gaussian integrals computation) and external concentration results are collected in Section~\ref{sec:appendix:technical}. 

\section{Setting}
\label{sec:appendix:setting}

Let us recall briefly the main assumptions under which we prove Theorem~\ref{th:main-result-light}. 
Recall that they are discussed in details in Section~\ref{sec:tabular-lime} of the main paper. 

\begin{hypo}[Linear $f$]
	\label{hyp:model}
	The black-box model can be written $a^\top x + b$, with $a\in\Reals^\Dim$ and $b\in\Reals$ fixed. 
\end{hypo}

\begin{hypo}[Gaussian sampling]
	\label{hyp:examples}
	The random variables $x_1,\ldots,x_n$ are i.i.d. $\gaussian{\mu}{\sigma^2\Id_\Dim}$. 
\end{hypo}

Also recall that, for any $1\leq i\leq n$, we set the weights to
\begin{equation}
\label{eq:def-weights-appendix}
\pi_i \defeq \exp\left(\frac{-\norm{x_i-\xi}^2}{2\nu^2}\right)
\, .
\end{equation}
We will need the following scaling constant:
\begin{equation}
\label{eq:def-cst}
\cst \defeq \left(\frac{\nu^2}{\nu^2+\sigma^2}\right)^{\Dim/2} \cdot \exp\left(\frac{-\norm{\xi-\mu}^2}{2(\nu^2+\sigma^2)}\right)
\, ,
\end{equation}
which does not play any role in the final result. 
One can check that $\cst \to 1$ when $\nu \gg \sigma$, regardless of the dimension. 

Finally, for any $1\leq j\leq \Dim$, recall that we defined
\begin{equation}
\label{eq:def-alpha}
\alpha_j \defeq \left[\frac{1}{2}\erfun{\frac{x-\mutilde_j}{\sigmatilde\sqrt{2}}}\right]_{q_{j-}}^{q_{j+}}
\, ,
\end{equation}
and
\begin{equation}
\theta_j \defeq \left[ \frac{\sigmatilde}{\sqrt{2\pi}}\exp\left(\frac{-(x-\mutilde_j)^2}{2\sigmatilde^2}\right)\right]_{q_{j-}}^{q_{j+}}
\, ,
\end{equation}
where $q_{j\pm}$ are the quantile boundaries of $\xi_j$. 
These coefficients are discussed in Section~\ref{sec:proofs} of the main paper. 
Note that all the expected values are taken with respect to the randomness on the $x_1,\ldots,x_n$. 

\section{Covariance matrix}
\label{sec:appendix:covariance}

In this section, we state and prove the intermediate results used to control the covariance matrix $\Sigmahat$. 
The goal of this section is to obtain the control of $\opnorm{\Sigmahatinv-\Sigmainv}$ in probability. 
Intuitively, if this quantity is small enough, then we can inverse Eq.~\eqref{eq:weighted-least-squares-main} and make very precise statements about~$\betahat$.

We first show that it is possible to compute the expected covariance matrix in closed form. 
Without this result, a concentration result would still hold, but it would be much harder to gain precise insights on the $\beta_j$s.

\begin{lemma}[Expected covariance matrix]
\label{lemma:expected-covariance-matrix}
Under Assumption~\ref{hyp:examples}, the expected value of $\Sigmahat$ is given by 
\[
\Sigma \defeq \cst\begin{pmatrix}
1 & \alpha_1 & \cdots & \alpha_\Dim \\
\alpha_1 & \alpha_1 & & \alpha_i\alpha_j \\
\vdots & & \ddots & \\
\alpha_\Dim & \alpha_i\alpha_j & & \alpha_\Dim 
\end{pmatrix}
\, .
\]
\end{lemma}

\begin{proof}
Elementary computations yield
\begin{equation*}
%\label{eq:sigma-hat}
\Sigmahat = \frac{1}{n}\begin{pmatrix}
\sum_{i=1}^n \pi_i & \sum_{i=1}^n \pi_i z_{i1} & \cdots & \sum_{i=1}^n \pi_i z_{i\Dim} \\
\sum_{i=1}^n \pi_i z_{i1} & \sum_{i=1}^n \pi_i z_{i1} & & \sum_{i=1}^n \pi_i z_{ik}z_{i\ell} \\
\vdots & & \ddots & \\
\sum_{i=1}^n \pi_i z_{i\Dim} & \sum_{i=1}^n \pi_i z_{ik}z_{i\ell} & & \sum_{i=1}^n \pi_i z_{i\Dim}
\end{pmatrix}
\, .
\end{equation*}

Reading the coefficients of this matrix, we have essentially three computations to complete: $\expec{\pi_i}$, $\expec{\pi_i z_{ik}}$, and $\expec{\pi_i z_{ik}z_{i\ell}}$. 
%Let us begin by computing $\expec{\pi_i}$. 

\paragraph{Computation of $\expec{\pi_i}$. }
Since the $x_i$s are Gaussian (Assumption~\ref{hyp:examples}) and using the definition of the weights (Eq.~\eqref{eq:def-weights-appendix}), we can write
\[
\expec{\pi_i} = \int_{\Reals^\Dim} \exp\left(\frac{-\norm{x_i-\xi}^2}{2\nu^2}\right)\exp\left(\frac{-\norm{x_i-\mu}^2}{2\sigma^2}\right) \frac{\diff x_{i1}\cdots x_{i\Dim}}{(2\pi\sigma^2)^{\Dim/2}}
\, .
\]
By independence across coordinates, the last display amounts to 
\[
\prod_{j=1}^\Dim \int_{-\infty}^{+\infty} \exp\left(\frac{-(x-\xi_j)^2}{2\nu^2}+ \frac{-(x-\mu_j)^2}{2\sigma^2}\right) \frac{\diff x}{\sigma\sqrt{2\pi}}
\, .
\]
We then apply Lemma~\ref{lemma:gaussian-integral-zero} to each of the integrals within the product to obtain
\[
\prod_{j=1}^\Dim \frac{\nu}{\sqrt{\nu^2+\sigma^2}} \cdot  \exp\left(\frac{-(\xi_j-\mu_j)^2}{2(\nu^2+\sigma^2)}\right) = \frac{\nu^\Dim}{(\nu^2+\sigma^2)^{\Dim / 2}} \cdot \exp\left(\frac{-\norm{\xi-\mu}^2}{2(\nu^2+\sigma^2)}\right)
\, .
\]
We recognize the definition of the scaling constant (Eq.~\eqref{eq:def-cst}): we have proved that $\expec{\pi_i} = \cst$.

\paragraph{Computation of $\expec{\pi_i z_{ik}}$. }
Since the $x_i$s are Gaussian (Assumption~\ref{hyp:examples}) and using the definition of the weights (Eq.~\eqref{eq:def-weights-appendix}), 
\[
\expec{\pi_i} = \int_{\Reals^\Dim} \exp\left(\frac{-\norm{x_i-\xi}^2}{2\nu^2}\right)\exp\left(\frac{-\norm{x_i-\mu}^2}{2\sigma^2}\right)\indic{\phi(x_i)_k=\phi(\xi)_k} \frac{\diff x_{i1}  \cdots x_{i\Dim}}{(2\pi\sigma^2)^{\Dim/2}}
\, .
\]
By independence across coordinates, the last display amounts to 
\[
\int_{q_{k-}}^{q_{k+}} \exp\left(\frac{-(x-\xi_k)^2}{2\nu^2}+ \frac{-(x-\mu_k)^2}{2\sigma^2}\right) \frac{\diff x}{\sigma\sqrt{2\pi}} \cdot \prod_{\substack{j=1\\j\neq k}}^\Dim \int_{-\infty}^{+\infty} \exp\left(\frac{-(x-\xi_j)^2}{2\nu^2}+ \frac{-(x-\mu_j)^2}{2\sigma^2}\right) \frac{\diff x}{\sigma\sqrt{2\pi}}
\, .
\]
Using Lemma~\ref{lemma:gaussian-integral-zero}, we obtain
\[
\frac{\nu^\Dim}{(\nu^2+\sigma^2)^{\Dim / 2}} \cdot \exp\left(\frac{-\norm{\xi-\mu}^2}{2(\nu^2+\sigma^2)}\right) \cdot \left[\frac{1}{2}\erfun{\frac{\nu^2 (x-\mu_k) + \sigma^2 (x-\xi_k)}{\nu\sigma\sqrt{2(\nu^2+\sigma^2)}}}\right]_{q_{k-}}^{q_{k+}}
\, .
\]
We recognize the definition of the scaling constant (Eq.~\eqref{eq:def-cst}) and of the $\alpha_k$ coefficient (Eq.~\eqref{eq:def-alpha}): we have proved that $\expec{\pi_i z_{ik}} = \cst \alpha_k$. 

\paragraph{Computation of $\expec{\pi_i z_{ik}z_{i\ell}}$. }
Since the $x_i$s are Gaussian (Assumption~\ref{hyp:examples}) and using the definition of the weights (Eq.~\eqref{eq:def-weights-appendix}), 
\[
\expec{\pi_iz_{ik}z_{i\ell}} = \int_{\Reals^\Dim} \exp\left(\frac{-\norm{x_i-\xi}^2}{2\nu^2}\right)\exp\left(\frac{-\norm{x_i-\mu}^2}{2\sigma^2}\right)\indic{\phi(x_i)_k=\phi(\xi)_k}\indic{\phi(x_i)_\ell=\phi(\xi)_\ell} \frac{\diff x_{i1}  \cdots \diff x_{i\Dim}}{(2\pi\sigma^2)^{\Dim/2}}
\, .
\]
By independence across coordinates, the last display amounts to 
\begin{align*}
\prod_{\substack{j=1\\j\neq k,\ell}}^\Dim \int_{-\infty}^{+\infty} \exp\left(\frac{-(x-\xi_j)^2}{2\nu^2}+ \frac{-(x-\mu_j)^2}{2\sigma^2}\right) \frac{\diff x}{\sigma\sqrt{2\pi}} & \cdot \int_{q_{k-}}^{q_{k+}} \exp\left(\frac{-(x-\xi_k)^2}{2\nu^2}+ \frac{-(x-\mu_k)^2}{2\sigma^2}\right) \frac{\diff x}{\sigma\sqrt{2\pi}} \\ &\cdot \int_{q_{\ell-}}^{q_{\ell+}} \exp\left(\frac{-(x-\xi_\ell)^2}{2\nu^2}+ \frac{-(x-\mu_\ell)^2}{2\sigma^2}\right) \frac{\diff x}{\sigma\sqrt{2\pi}}
\, .
\end{align*}
Using Lemma~\ref{lemma:gaussian-integral-zero}, we obtain
\begin{align*}
\frac{\nu^\Dim}{(\nu^2+\sigma^2)^{\Dim / 2}} \cdot \exp\left(\frac{-\norm{\xi-\mu}^2}{2(\nu^2+\sigma^2)}\right) & \cdot \left[\frac{1}{2}\erfun{\frac{\nu^2 (x-\mu_k) + \sigma^2 (x-\xi_k)}{\nu\sigma\sqrt{2(\nu^2+\sigma^2)}}}\right]_{q_{k-}}^{q_{k+}} \\
&\cdot \left[\frac{1}{2}\erfun{\frac{\nu^2 (x-\mu_\ell) + \sigma^2 (x-\xi_\ell)}{\nu\sigma\sqrt{2(\nu^2+\sigma^2)}}}\right]_{q_{\ell-}}^{q_{\ell+}}
\, .
\end{align*}
We recognize the definition of the scaling constant (Eq.~\eqref{eq:def-cst}) and of the alphas (Eq.~\eqref{eq:def-alpha}): we have proved that $\expec{\pi_i z_{ik}z_{i\ell}} = \cst \alpha_k\alpha_{\ell}$. 
\end{proof}

As it turns out, we show that it is possible to invert $\Sigma$ in closed-form, therefore simplifying tremendously our quest for control of $\opnorm{\Sigmahatinv-\Sigmainv}$. 
% TODO: already said
Indeed, in most cases, even if concentration could be shown, one would not have a precise idea of the coefficients of $\Sigmainv$. 

\begin{lemma}[Inverse of the covariance matrix]
\label{lemma:inverse-covariance-matrix}
If $\alpha_j\neq 0,1$ for any $j\in\{1,\ldots,\Dim\}$, then $\Sigma$ is invertible, and
\[
\Sigmainv = \cst^{-1}\begin{pmatrix}
1+\sum_{j=1}^\Dim \frac{\alpha_j}{1-\alpha_j} & \frac{-1}{1-\alpha_1} & \cdots & \frac{-1}{1-\alpha_\Dim} \\
\frac{-1}{1-\alpha_1} & \frac{1}{\alpha_1(1-\alpha_1)} & & 0 \\
\vdots & & \ddots & \\
\frac{-1}{1-\alpha_\Dim} & 0 & & \frac{1}{\alpha_\Dim(1-\alpha_\Dim)}
\end{pmatrix}
\, .
\]
\end{lemma}

\begin{proof}
Define $\alpha\in\Reals^{\Dim}$ the vector of the $\alpha_j$s. 
Set $A\defeq 1$, $B\defeq \alpha^\top$, $C\defeq\alpha$, and
\[
D \defeq \begin{pmatrix}
\alpha_1 & & \alpha_j\alpha_k \\
 & \ddots & \\
\alpha_j\alpha_k & & \alpha_\Dim 
\end{pmatrix}
\, .
\]
Then $\Sigma$ is a block matrix that can be written $\Sigma = \cst \begin{bmatrix} A & B \\ C & D\end{bmatrix}$. 
We notice that 
\[
D-CA^{-1}B = \Diag\left(\alpha_1(1-\alpha_1),\ldots,\alpha_\Dim(1-\alpha_\Dim)\right)
\, .
\]
%
%is an invertible matrix since the $\alpha_j$s are not~$0$ or~$1$. 
Note that, since~$\Erfun$ is an increasing function, the $\alpha_j$s are always distinct from~$0$ and~$1$. 
%When $\nu \gg 1$, $\mutilde\to\mu$ and $\sigmatilde\to\sigma$, and by definition of the quantiles, $\alpha_j\to 1/p$. 
%The $\alpha_i$s can be seen as a generalization of this $1/p$ weight, depending on the exact localization of~$\xi_j$ with respect to~$\mu_j$, and the parameters~$\nu$ and~$\sigma$. 
Thus $D-CA^{-1}B$ is an invertible matrix, and we can use the block matrix inversion formula to obtain the claimed result. 
\end{proof}

As a direct consequence of the computation of $\Sigmainv$, we can control its largest eigenvalue.

\begin{lemma}[Control of $\opnorm{\Sigmainv}$]
\label{lemma:control-largest-eigenvalue}
We have the following bound on the operator norm of the inverse covariance matrix: 
\[
\opnorm{\Sigmainv} \leq \frac{3\Dim \alphacst}{\cst}
\, ,
\]
where $\alphacst \defeq \max_{1\leq j\leq \Dim} \frac{1}{\alpha_j(1-\alpha_j)}$.
\end{lemma}

\begin{proof}
We control the operator norm of $\Sigmainv$ by its Frobenius norm:
%, since $\Sigmainv$ has only $\Dim$ non-zero elements. 
Namely, 
\begin{align*}
\opnorm{\Sigmainv}^2 &\leq \frobnorm{\Sigmainv}^2 \\
&= \cst^{-2} \left[\left(1+\sum \frac{\alpha_j}{1-\alpha_j}\right)^2 + \sum \frac{1}{(1-\alpha_j)^2} + \sum \frac{1}{\alpha_j(1-\alpha_j)}\right] \\
\opnorm{\Sigmainv}^{2} &\leq 6 \cst^{-2} \Dim^{2} \left(\max \frac{1}{\alpha_j(1-\alpha_j)}\right)^{2}
\, ,
\end{align*}
where we used $\alpha_j\in (0,1)$ in the last step of the derivation. 
\end{proof}

\begin{remark}
Better bounds can without doubt be obtained. 
A step in this direction is to notice that $S\defeq \cst \Sigmainv$ is an arrowhead matrix \citep{OLe_Ste:1996}. 
Thus the eigenvalues of~$S$ are solutions of the secular equation 
\[
1+\sum_{j=1}^\Dim \frac{\alpha_j}{1-\alpha_j}-\lambda +\sum_{j=1}^\Dim \frac{\alpha_j}{(1-\alpha_j)(1-\lambda \alpha_j (1-\alpha_j))} = 0
\, .
\]
Further study of this equation could yield an improved statement for Lemma~\ref{lemma:control-largest-eigenvalue}. 
% NOTE: possibly without the d factor.
\end{remark}

We now show that the empirical covariance matrix concentrates around $\Sigma$. 
It is interesting to see that the non-linear nature of the new coordinates (the $z_{ij}$s) calls for complicated computations but allows us to use simple concentration tools since they are, in essence, Bernoulli random variables. 

\begin{lemma}[Concentration of the empirical covariance matrix]
\label{lemma:concentration-empirical-covariance-matrix}
Let $\Sigmahat$ and $\Sigma$ be defined as before. 
Then, for every $t>0$,
\[
\proba{\opnorm{\Sigmahat - \Sigma} \geq t} \leq 4\Dim^2 \exp (-2nt^2)
\, .
\]
\end{lemma}

\begin{proof}
Recall that $\opnorm{\cdot}\leq \frobnorm{\cdot}$: it suffices to show the result for the Frobenius norm. 
Next, we notice that the summands appearing in the entries of $\Sigmahat$, $X_i^{(1)}\defeq \pi_i$, $X_i^{(2,k)}\defeq \pi_i z_{ik}$, and $X_i^{(3,k,\ell)}\defeq \pi_i z_{ik}z_{i\ell}$, are all bounded. 
Indeed, by the definition of the weights and the definition of the new features, they all take values in $[0,1]$. 
Moreover, for given $k,\ell$, they are independent random variables. 
Thus we can apply Hoeffding's inequality (Theorem~\ref{th:hoeffding}) to $X_i^{(1)}$, $X_i^{(2,k)}$, and $X_i^{(3,k,\ell)}$. 
For any given $t >0$, we obtain
\[
\begin{cases}
\proba{\abs{\frac{1}{n}\sum_{i=1}^n (\pi_i-\expec{\pi_i})} \geq t} \leq 2\exp(-2nt^2) \\ 
\proba{\abs{\frac{1}{n}\sum_{i=1}^n (\pi_i z_{ik}-\expec{\pi_i})} \geq t} \leq 2\exp(-2nt^2) \\ 
\proba{\abs{\frac{1}{n}\sum_{i=1}^n (\pi_i z_{ik}z_{i\ell}-\expec{\pi_i})} \geq t} \leq 2\exp(-2nt^2)
\end{cases}
\]
We conclude by a union bound on the $(\Dim+1)^2\leq 2\Dim^2$ entries of the matrix. 
\end{proof}

% TODO: remark on the dimension, talk about random matrices, Vershynin

As a consequence of the two preceding lemmas, we can control the largest eigenvalue of $\Sigmainv$. 

\begin{lemma}[Control of $\opnorm{\Sigmahatinv}$]
\label{lemma:control-largest-eigenvalue-empirical}
For every $t\in \left(0,\frac{\cst}{6\Dim \alphacst}\right]$, with probability greater than $1-4\Dim^2 \exp(-2nt^2)$, 
\[
\opnorm{\Sigmahatinv} \leq \frac{6\Dim \alphacst}{\cst} 
\, .
\]
\end{lemma}

\begin{proof}
Let $t \in (0,\cst/(6\Dim \alphacst)]$. 
According to Lemma~\ref{lemma:control-largest-eigenvalue}, $\lambda_{\max}(\Sigmainv) \leq 3\Dim \alphacst /\cst$. 
We deduce that 
\[
\lambda_{\min} (\Sigma) \geq \frac{\cst}{3\Dim \alphacst} %\geq 2t
\, .
\]
Now let us use Lemma~\ref{lemma:concentration-empirical-covariance-matrix} with this~$t$: 
there is an event~$\Omega$, which has probability greater than $1-4\Dim^2\exp(-2nt^2)$, such that $\opnorm{\Sigmahat-\Sigma} \leq t$. 
According to Weyl's inequality \citep{Wey:1912}, on this event, 
\[
\abs{\lambda_{\min}(\Sigmahat) - \lambda_{\min}(\Sigma)} \leq \opnorm{\Sigmahat-\Sigma} \leq t
\, .
\]
In particular, 
\[
\lambda_{\min}(\Sigmahat) \geq \lambda_{\min}(\Sigma) - t \geq \frac{\cst}{6\Dim \alphacst}
\, .
\]
Finally, we deduce that
\[
\opnorm{\Sigmahatinv} \leq \frac{6\Dim \alphacst}{\cst} 
\, .
\]
\end{proof}

We can now state and prove the main result of this section, controlling the operator norm of $\Sigmahat - \Sigma$ with high probability.

\begin{proposition}[Control of $\opnorm{\Sigmahatinv - \Sigmainv}$]
\label{prop:control-operator-norm-inverse-covariance}
For every $t\in\left(0,\frac{3\Dim \alphacst}{\cst}\right]$, we have
\[
\proba{\opnorm{\Sigmahatinv - \Sigmainv} \geq t} \leq 8\Dim^2 \exp\left(\frac{-\cst^4 nt^2}{162\Dim^4\alphacst^4}\right)
\, .
\]
\end{proposition}

\begin{remark}
Proposition~\ref{prop:control-operator-norm-inverse-covariance} is the key tool to invert Eq.~\eqref{eq:weighted-least-squares-main} and gain precise control over $\betahat$. 
In the regime that we consider, the dimension $\Dim$ as well as the number of bins $p$ are \emph{fixed}, and $\Dim, \cst$, and $\alphacst$ are essentially numerical constants. 
We did not optimize these constant with respect to $\Dim$, since the main message is to consider the behavior for a large number of new examples ($n\to +\infty$).
\end{remark}

\begin{proof}
We notice that, assuming that $\Sigmahat$ is invertible, $\Sigmahatinv-\Sigmainv = \Sigmahatinv (\Sigma - \Sigmahat)\Sigmainv$. 
Since $\opnorm{\cdot}$ is sub-multiplicative, we just have to control each term individually. 
Lemma~\ref{lemma:control-largest-eigenvalue} gives us 
\[
\opnorm{\Sigmainv} \leq \frac{3\Dim \alphacst}{\cst} 
\, .
\]
Next, set $t_1\defeq \frac{\cst^2 t}{18\Dim^2 \alphacst^2}$. 
According to Lemma~\ref{lemma:concentration-empirical-covariance-matrix}, with probability greater than $1-4\Dim^2\exp(-2nt_1^2)$, 
\[
\opnorm{\Sigmahat - \Sigma} \leq t_1
\, .
\]
Finally, set $t_2 \defeq t_1$. 
It is easy to check that $t_2 \leq \cst/(6\Dim \alphacst)$. 
Thus we can use Lemma~\ref{lemma:control-largest-eigenvalue-empirical}: with probability greater than $1-4\Dim^2\exp(-2nt_1^2)$, 
\[
\opnorm{\Sigmahatinv} \leq \frac{6\Dim \alphacst}{\cst}
\, .
\]
By the union bound, with probability greater than $1-8\Dim^2 \exp\left(\frac{-\cst^4 nt^2}{162\Dim^4\alphacst^4}\right)$, 
\begin{align*}
\opnorm{\Sigmahatinv - \Sigmahat} &\leq \opnorm{\Sigmainv} \cdot \opnorm{\Sigmahat - \Sigma} \cdot \opnorm{\Sigmahatinv} \\
&\leq \frac{3\Dim \alphacst}{\cst} \cdot t_1 \cdot \frac{6\Dim\alphacst}{\cst} = t
\, .
\end{align*}
\end{proof}

\section{Right-hand side of Eq.~\eqref{eq:weighted-least-squares-main}}
\label{sec:appendix:response-vector}

In this section, we state and prove the results in relation to~$\Gammahat$. 
We begin with the computation of~$\Gamma$, the expected value of~$\Gammahat$. 

\begin{lemma}[Computation of $\Gamma$]
\label{lemma:computation-expected-response-vector}
Under Assumption~\ref{hyp:examples} and~\ref{hyp:model}, the expected value of $\Gammahat$ is given by
\[
\Gamma = \cst \begin{pmatrix}
f(\mutilde) \\
\alpha_1 f(\mutilde) - a_1 \theta_1 \\
\vdots \\
\alpha_\Dim f(\mutilde) - a_\Dim \theta_\Dim 
\end{pmatrix}
\, ,
\]
where the $\theta_j$s are defined by
\[
\theta_j \defeq \left[ \frac{\sigmatilde}{\sqrt{2\pi}}\exp\left(\frac{-(x-\mutilde_j)^2}{2\sigmatilde^2}\right)\right]_{q_{j-}}^{q_{j+}}
\, .
\]
\end{lemma}

\begin{proof}
Given the expression of $\Gammahat$, we have essentially two computations to manage: $\expec{\pi_i f(x_i)}$ and $\expec{\pi_i z_{ij}f(x_i)}$. 

\paragraph{Computation of $\expec{\pi_i f(x_i)}$. }

Under Assumption~\ref{hyp:model}, by linearity of the integral, 
\begin{equation}
\label{eq:computation-expected-response-vector-aux-1}
\expec{\pi_i f(x_i)} = \expec{\pi_i (a^\top+b)} = b\expec{\pi_i} + \sum_{j=1}^\Dim a_{j} \expec{\pi_i x_{ij}}
\, .
\end{equation}
Now we have already seen in the proof of Lemma~\ref{lemma:expected-covariance-matrix} that $\expec{\pi_i} = \cst$. 
Thus we can focus on the computation of $\expec{\pi_i x_{ij}}$ for fixed $i,j$. 
Under Assumption~\ref{hyp:examples}, we have
\[
\expec{\pi_i x_{ij}} = \int_{\Reals^\Dim} x_j\cdot \exp\left(\frac{-\norm{x-\xi}^2}{2\nu^2}+\frac{-\norm{x-\mu}^2}{2\sigma^2}\right) \frac{\diff x_1 \cdots \diff x_\Dim}{(2\pi\sigma^2)^{\Dim/2}}
\, .
\]
By independence, the last display amounts to
\[
\int_{-\infty}^{+\infty} x\cdot \exp\left(\frac{-(x-\xi_j)^2}{2\nu^2}+\frac{-(x-\mu_j)^2}{2\sigma^2}\right) \frac{\diff x}{\sigma\sqrt{2\pi}} \cdot 
\prod_{k\neq j} \int_{-\infty}^{+\infty} \exp\left(\frac{-(x-\xi_k)^2}{2\nu^2}+\frac{-(x-\mu_k)^2}{2\sigma^2}\right) \frac{\diff x}{\sigma\sqrt{2\pi}}
\, .
\]
A straightforward application of Lemmas~\ref{lemma:gaussian-integral-zero} and~\ref{lemma:gaussian-integral-first-order} yields
\[
\expec{\pi_i x_{ij}} = \cst \cdot \frac{\nu^2 \mu_j + \sigma^2\xi_j}{\nu^2+\sigma^2}
\, .
\]
Back to Eq.~\eqref{eq:computation-expected-response-vector-aux-1}, we have shown that
\[
\expec{\pi_i f(x_i)} = \cst b + \sum_{j=1}^\Dim a_j \cdot \cst \frac{\nu^2 \mu_j + \sigma^2\xi_j}{\nu^2+\sigma^2} = \cst f(\mutilde)
\, .
\]

\paragraph{Computation of $\expec{\pi_i z_{ij}f(x_i})$. }

Under Assumption~\ref{hyp:model}, by linearity of the integral, 
\begin{equation}
\label{eq:computation-expected-response-vector-aux-2}
\expec{\pi_i z_{ij}f(x_i)} = b\expec{\pi_i z_{ij}} + \sum_{k=1}^\Dim a_k \cdot \expec{\pi_i z_{ij} x_{ik}}
\, .
\end{equation}
We have already computed $\expec{\pi_i z_{ij}}$ in the proof of Lemma~\ref{lemma:expected-covariance-matrix} and found that 
\[
\expec{\pi_i z_{ij}} = \cst \alpha_j
\, .
\]
Regarding the computation of $\expec{\pi_i z_{ij} x_{ik}}$, there are essentially two cases to consider depending whether $k=\ell$ or not. 
Let us first consider the case $k=j$. 
Then we obtain
\[
\expec{\pi_i z_{ij} x_{ik}} = \int_{\Reals^\Dim} x_{j} \exp\left(\frac{-\norm{x-\xi}^2}{2\nu^2}+\frac{-\norm{x-\mu}^2}{2\sigma^2}\right)\indic{\phi(x)_j=\phi(\xi)_j} \frac{\diff x_1 \cdots \diff x_\Dim}{(2\pi\sigma^2)^{\Dim/2}}
\, .
\]
By independence, the last display amounts to
\[
\int_{q_{j-}}^{q_{j+}} x \cdot \exp\left(\frac{-(x-\xi_j)^2}{2\nu^2}+\frac{-(x-\mu_j)^2}{2\sigma^2}\right) \frac{\diff x}{\sigma\sqrt{2\pi}} \cdot 
\prod_{k\neq j} \int_{-\infty}^{+\infty} \exp\left(\frac{-(x-\xi_k)^2}{2\nu^2}+\frac{-(x-\mu_k)^2}{2\sigma^2}\right) \frac{\diff x}{\sigma\sqrt{2\pi}}
\, .
\]
According to Lemma~\ref{lemma:gaussian-integral-first-order} and the definition of $\alpha_j$ and $\theta_j$ (Eqs.~\eqref{eq:def-alpha} and~\eqref{eq:def-alpha}), we have
\[
\expec{\pi_i z_{ij} x_{ij}} = \cst \frac{\sigma^2\xi_j+\nu^2\mu_j}{\nu^2+\sigma^2} \alpha_j - \cst \theta_j
\, .
\]
Now if $k\neq j$, by independence, $\expec{\pi_i z_{ij} x_{ik}}$ splits in three parts:
\begin{align*}
\expec{\pi_i z_{ij} x_{ik}} &= \int_{-\infty}^{+\infty} x\cdot \exp\left(\frac{-(x-\xi_k)^2}{2\nu^2}+\frac{-(x-\mu_k)^2}{2\sigma^2}\right) \frac{\diff x}{\sigma\sqrt{2\pi}} \cdot \int_{q_{j-}}^{q_{j+}} \exp\left(\frac{-(x-\xi_j)^2}{2\nu^2}+\frac{-(x-\mu_j)^2}{2\sigma^2}\right) \frac{\diff x}{\sigma\sqrt{2\pi}}\cdot \\
&\cdot \prod_{\ell \neq j,k}  \int_{-\infty}^{+\infty} \exp\left(\frac{-(x-\xi_k)^2}{2\nu^2}+\frac{-(x-\mu_k)^2}{2\sigma^2}\right) \frac{\diff x}{\sigma\sqrt{2\pi}}
\, .
\end{align*}
Lemma~\ref{lemma:gaussian-integral-zero} and~\ref{lemma:gaussian-integral-first-order} yield
\[
\expec{\pi_i z_{ij} x_{ik}} = \cst \cdot \frac{\sigma^2\xi_k+\nu^2\mu_k}{\nu^2+\sigma^2} \cdot \alpha_j
\, .
\]
In definitive, plugging these results into Eq.~\eqref{eq:computation-expected-response-vector-aux-2} gives
\begin{align*}
\expec{\pi_i z_{ij}f(x_i)} &= \cst \alpha_j b + a_j\left(\cst \frac{\sigma^2\xi_j+\nu^2\mu_j}{\nu^2+\sigma^2} \alpha_j - \cst \theta_j\right) + \sum_{k\neq j} a_k \cdot \cst \frac{\sigma^2\xi_k+\nu^2\mu_k}{\nu^2+\sigma^2} \alpha_j \\
&= \cst \alpha_j f(\mutilde) - \cst a_j \theta_j
\, .
\end{align*}
\end{proof}

As a consequence of Lemma~\ref{lemma:computation-expected-response-vector}, we can control $\norm{\Gamma}$. 

\begin{lemma}[Control of $\norm{\Gamma}$]
\label{lemma:control-norm-gamma}
Under Assumptions~\ref{hyp:examples} and~\ref{hyp:model}, it holds that
\[
\norm{\Gamma}^2 \leq \cst^2\left(3\Dim f(\mutilde)^2 + d\sigmatilde^2 \norm{\nabla f}^2\right)
\, .
\]
\end{lemma}

\begin{proof}
According to Lemma~\ref{lemma:computation-expected-response-vector}, we have
\[
\norm{\Gamma}^2 = \cst^2\left(f(\mutilde)^2 + \sum_{j=1}^{\Dim} (\alpha_j f(\mutilde) - a_j\theta_j)^2\right)
\, .
\]
Successively using $(x-y)^2\leq 2(x^2+y^2)$, $\alpha_j\in[0,1]$ and $\theta_j\in [-\sigmatilde/\sqrt{2\pi},\sigmatilde/\sqrt{2\pi}]$, we write
\begin{align*}
\norm{\Gamma}^2 &\leq \cst^2\left(f(\mutilde)^2 + \sum_{j=1}^{\Dim} 2(\alpha_j^2 f(\mutilde)^2 + a_j^2\theta_j^2)\right) \\
&\leq \cst^2\left(3\Dim f(\mutilde)^2 + d\sigmatilde^2 \norm{a}^2\right)
\, ,
\end{align*}
which concludes the proof. 
\end{proof}

Finally, we conclude this section with a concentration result for $\Gammahat$. 

\begin{lemma}[Concentration of $\norm{\Gammahat}$]
\label{lemma:concentration-response}
Under Assumptions~\ref{hyp:examples} and~\ref{hyp:model}, for any $t>0$, we have
\[
\proba{\norm{\Gammahat - \Gamma} > t} \leq 4\Dim \exp\left(\frac{-nt^2}{2\norm{\nabla f}^2 \sigma^2}\right)
\, .
\]
\end{lemma}

\begin{proof}
Since the $x_i$ are Gaussian with variance $\sigma^2$ (Assumption~\ref{hyp:examples}), the random variable $a^\top x_i+b$ is Gaussian with variance $\norm{a}^2\sigma^2$, and the $X_i^{(1)}\defeq \pi_i x_{i}$ are sub-Gaussian with parameter $\norm{a}^2\sigma^2$. 
They are also independent, thus we can apply Theorem~\ref{th:hoeffding-subgaussian} to the $X_i^{(1)}$:
\[
\proba{\abs{\frac{1}{n}\sum_{i=1}^{n} \pi_if(x_i) - \expec{\pi_i f(x_i)}} > t} \leq 2\exp\left(\frac{-nt^2}{2\norm{a}^2 \sigma^2}\right)
\, .
\]
Furthermore, the $z_{ij}$ are $\{0,1\}$-valued. 
Thus the random variables $X_i^{(j)}\defeq \pi_i z_{ij}f(x_i)$ are also sub-Gaussian with parameter $\norm{a}^2\sigma^2$. 
We use Hoeffding's inequality (Theorem~\ref{th:hoeffding-subgaussian}) again, to obtain, for any $j$, 
\[
\proba{\abs{\frac{1}{n}\sum_{i=1}^{n} \pi_i z_{ij}f(x_i) - \expec{\pi_i z_{ij}f(x_i)}} > t} \leq 2\exp\left(\frac{-nt^2}{2\norm{a}^2 \sigma^2}\right)
\, .
\]
By the union bound,
\[
\proba{\norm{\Gammahat - \Gamma} > t} \leq 2(\Dim+1) \exp\left(\frac{-nt^2}{2\norm{a}^2 \sigma^2}\right)
\, .
\]
We deduce the result since $\Dim \geq 1$.
\end{proof}

\section{Proof of the main result}
\label{sec:appendix:conclusion}

In this section, we state and prove our main result, Theorem~\ref{th:main-result}. 
It is a more precise version than Theorem~\ref{th:main-result-light} in the main paper. 

%First, we compute a quantity of interest. 

%\begin{proof}
%Take the expression of $\Sigmainv$ given by Lemma~\ref{lemma:inverse-covariance-matrix} and the expression of $\Gamma$ given by Lemma~\ref{lemma:computation-expected-response-vector}. 
%\end{proof}

%Now we state our main result. 

\begin{theorem}[Concentration of $\betahat$]
\label{th:main-result}
Let $\eta \in (0,1)$ and $\epsilon >0$. 
Take
% TODO: clean this a bit
\[
n \geq \max\left(\frac{288\norm{\nabla f}^2\sigma^2\Dim^2\alphacst^2}{\epsilon^2\cst^2}\log \frac{12\Dim}{\eta}, \frac{18\Dim^2\alphacst^2}{\cst^2}\log \frac{24\Dim^2}{\eta},\frac{648\Dim^5\alphacst^4(3f(\mutilde)^2+\sigmatilde^2\norm{\nabla f}^2)}{\cst^2\epsilon^2}\log \frac{24\Dim^2}{\eta}\right) 
\, .
\]
Then, under assumptions~\ref{hyp:examples} and~\ref{hyp:model}, 
\[
\norm{\betahat - \Sigmainv \Gamma} \leq \epsilon
\, ,
\]
with probability greater than $1-\eta$.
\end{theorem}

\begin{proof}
The main idea of the proof is to notice that
\begin{align}
\norm{\betahat - \Sigmainv \Gamma} &= \norm{\Sigmahatinv \Gammahat - \Sigmainv\Gamma} \notag \\
&\leq \norm{\Sigmahatinv (\Gammahat - \Gamma)} + \norm{(\Sigmahatinv-\Sigmainv)\Gamma} \notag
\, ,
\end{align}
and then to control these two terms using the results of Section~\ref{sec:appendix:covariance} and~\ref{sec:appendix:response-vector}. 

\paragraph{Control of $\norm{\Sigmahatinv (\Gammahat - \Gamma)}$. }

We use the upper bound $\norm{\Sigmahatinv (\Gammahat - \Gamma)} \leq \opnorm{\Sigmahatinv}\cdot  \norm{\Gammahat-\Gamma}$. 
We then achieve control of the operator norm of the empirical covariance matrix in probability with Lemma~\ref{lemma:control-largest-eigenvalue-empirical}, and control of the norm of $\Gammahat-\Gamma$ in probability with Lemma~\ref{lemma:concentration-response}. 
Set 
\[
t_1\defeq \frac{\cst}{6\Dim \alphacst}\quad\text{and}\quad  n_1 \defeq  \frac{18\Dim^2}{\cst^2}\log \frac{12\Dim^2}{\eta}
\, .
\]
According to Lemma~\ref{lemma:control-largest-eigenvalue-empirical}, for any $n\geq n_1$, there is an event $\Omega_1^n$ which has probability greater than $1-4\Dim^2\exp (-2nt_1^2)$ such that 
\[
\opnorm{\Sigmahatinv} \leq \frac{6\Dim\alphacst}{\cst}
\]
on this event. 
It is easy to check that $4\Dim^2\exp(-2n_1t_1^2) = \eta / 3$, thus $\Omega_1^n$ has probability greater than $1-\eta / 3$. 
Now set
\[
t_2 \defeq \frac{\epsilon\cst}{12\Dim\alphacst} \quad\text{and}\quad 
 n_2 \defeq \frac{288\norm{a}^2\sigma^2\Dim^2\alphacst^2 }{\epsilon^2\cst^2}\log \frac{12\Dim}{\eta}
\, .
\]
According to Lemma~\ref{lemma:concentration-response}, for any $n\geq n_2$, there exists an event $\Omega_2^n$ which has probability greater than $1-4\Dim\exp\left(\frac{-nt_2^2}{2\norm{a}^2\sigma^2}\right)$ such that $\norm{\Gammahat-\Gamma} \leq t_2$ on that event.
One can check that
\[
4\Dim \exp\left(\frac{-n_2t_2^2}{2\norm{a}^2\sigma^2}\right) = \frac{\eta}{3}
\, ,
\]
thus $\Omega_2^n$ has probability greater than $1-\eta/3$. 
On the event $\Omega_1^n\cap\Omega_2^n$, we have
\[
\norm{\Sigmahatinv (\Gammahat - \Gamma)} \leq \opnorm{\Sigmahatinv}\cdot \norm{\Gammahat-\Gamma} \leq \frac{6\Dim\alphacst}{\cst} \cdot t_2 \leq \frac{\epsilon}{2}
\, , 
\]
by definition of $t_2$. 

\paragraph{Control of $\norm{(\Sigmahatinv-\Sigmainv)\Gamma}$. }

We use the upper bound $\norm{(\Sigmahatinv-\Sigmainv)\Gamma} \leq \opnorm{\Sigmahatinv-\Sigmainv} \cdot \norm{\Gamma}$. 
We then achieve control of $\opnorm{\Sigmahatinv-\Sigmainv}$ in probability with Proposition~\ref{prop:control-operator-norm-inverse-covariance}, whereas we can bound the norm of $\Gamma$ almost surely with Lemma~\ref{lemma:control-norm-gamma}. 
If $\norm{\Gamma}=0$, then there is nothing to prove. 
Otherwise, set
\[
t_3 \defeq \min\left(\frac{\epsilon}{2\norm{\Gamma}},\frac{3\Dim\alphacst}{\cst}\right),\,
n_3 \defeq \frac{18\Dim^2\alphacst^2}{\cst^2} \log \frac{24\Dim^2}{\eta},\, \quad \text{and} \quad
n_4 \defeq \frac{648\Dim^5\alphacst^4(3f(\mutilde)^2 + \sigmatilde^2\norm{a}^2)}{\cst^2\epsilon^2} \log \frac{24\Dim^2}{\eta}
\, . 
\]
According to Proposition~\ref{prop:control-operator-norm-inverse-covariance}, for any $n\geq \max(n_3,n_4)$, there is an event $\Omega_3^n$ which has probability greater than $1-8\Dim^2\exp\left(\frac{-\cst^3 nt_3^2}{162\Dim^2\alphacst^4}\right)$ such that
\[
\opnorm{\Sigmahatinv-\Sigmainv} \leq t_3
\]
on this event. 
With the help of Lemma~\ref{lemma:control-norm-gamma}, one can check that
\[
\max\left(8\Dim^2\exp\left(\frac{-\cst^3 n_3t_3^2}{162\Dim^2\alphacst^4}\right),8\Dim^2\exp\left(\frac{-\cst^3 n_4t_3^2}{162\Dim^2\alphacst^4}\right)\right) \leq \frac{\eta}{3}
\, .
\]
Therefore, $\Omega_3^n$ has probability greater than $\eta / 3$ and, on this event,
\[
\norm{(\Sigmahatinv-\Sigmainv)\Gamma} \leq \opnorm{\Sigmahatinv-\Sigmainv} \cdot \norm{\Gamma} \leq t_3 \cdot \norm{\Gamma} \leq \frac{\epsilon}{2}
\, .
\]

\paragraph{Conclusion. }

Set $n \geq \max(n_i,i=1\ldots 4)$. 
Define $\Omega^n \defeq \Omega_1^n\cap\Omega_2^n\cap\Omega_3^n$, where the $\Omega_i^n$ are defined as before. 
According to the previous reasoning, on the event $\Omega^n$,  
\begin{align}
\norm{\betahat - \Sigmainv \Gamma} &= \norm{\Sigmahatinv \Gammahat - \Sigmainv\Gamma} \notag \\
&\leq \norm{\Sigmahatinv (\Gammahat - \Gamma)} + \norm{(\Sigmahatinv-\Sigmainv)\Gamma} \notag \\
&\leq \frac{\epsilon}{2}+\frac{\epsilon}{2} = \epsilon \notag
\, .
\end{align}
Moreover, the union bound gives $\proba{\Omega^n} \geq 1-\eta$. 
We conclude by noticing that $n_1$ is always smaller than $n_3$, thus we just have to require $n \geq \max(n_2,n_3,n_4)$, as in the statement of our result. 
\end{proof}

%%%%%%%%%%%%%%%%%%%%%%%%%%%%%%%%%%%%%%%%%%%%%%%%%%%%%%%%%%%%%%%%%%%%%%%%%%%%%%%%%%%%%%%%%%%%%%%%%%%%%%%%%%%%%%%
\section{Technical lemmas}
\label{sec:appendix:technical}

\subsection{Gaussian integrals}

In this section, we collect some Gaussian integral computations that are needed in our derivations. 
We provide succinct proof, since essentially any modern computer algebra system will provide these formulas. 
Our first result is for zero-th order Gaussian integral. 

\begin{lemma}[Gaussian integral, $0$-th order]
\label{lemma:gaussian-integral-zero}
Let $\xi,\mu$ be real numbers, and $\nu,\sigma$ be positive real numbers. 
Then, it holds that
\[
\int \exp\left(\frac{-(x-\xi)^2}{2\nu^2}+ \frac{-(x-\mu)^2}{2\sigma^2}\right) \frac{\diff x}{\sigma\sqrt{2\pi}} = \frac{\nu}{\sqrt{\nu^2+\sigma^2}} \cdot  \exp\left(\frac{-(\xi-\mu)^2}{2(\nu^2+\sigma^2)}\right) \cdot \frac{1}{2}\erfun{\frac{\nu^2 (x-\mu) + \sigma^2 (x-\xi)}{\nu\sigma\sqrt{2(\nu^2+\sigma^2)}}}
\, .
\]
In particular, 
\[
\int_{-\infty}^{+\infty} \exp\left(\frac{-(x-\xi)^2}{2\nu^2}+ \frac{-(x-\mu)^2}{2\sigma^2}\right) \frac{\diff x}{\sigma\sqrt{2\pi}} = \frac{\nu}{\sqrt{\nu^2+\sigma^2}} \cdot  \exp\left(\frac{-(\xi-\mu)^2}{2(\nu^2+\sigma^2)}\right)
\, .
\]
\end{lemma}

\begin{proof}
For any reals $a,b$, and $c$, it holds that
\[
\int \exps{-ax^2+bx+c} \diff x = \sqrt{\frac{\pi}{a}}\cdot \exps{\frac{b^2}{4a}+c} \cdot \frac{1}{2}\erfun{\frac{2ax-b}{2\sqrt{a}}}
\, .
\]
We apply this formula with $a = \frac{1}{2\nu^2}+\frac{1}{2\sigma^2}$, $b=\frac{\xi}{\nu^2}+\frac{\mu}{\sigma^2}$, and $c=-\left(\frac{\xi^2}{2\nu^2}+\frac{\mu^2}{\sigma^2}\right)$. 
We then notice that $b^2/(4a)+c = \frac{-(\xi-\mu)^2}{2(\nu^2+\sigma^2)}$ and
\[
\frac{2ax-b}{2\sqrt{a}} = \frac{\nu^2 (x-\mu) + \sigma^2 (x-\xi)}{\nu\sigma\sqrt{2(\nu^2+\sigma^2)}}
\, .
\]
\end{proof}

\begin{remark}
We often replace $\frac{\nu^2 (x-\mu) + \sigma^2 (x-\xi)}{\nu\sigma\sqrt{2(\nu^2+\sigma^2)}}$ by the more readable $(x-\mutilde)/(\sigmatilde\sqrt{2})$ in the main text of the paper.
\end{remark}

Since $f$ is assumed to be linear in most of the paper, we need first order computations as well:

\begin{lemma}[Gaussian integral, $1$st order]
\label{lemma:gaussian-integral-first-order}
Let $\xi,\mu$ be real numbers, and $\nu,\sigma$ be positive numbers. 
Then it holds that
\begin{align*}
\int & x\cdot \exp\left(\frac{-(x-\xi)^2}{2\nu^2}+\frac{-(x-\mu)^2}{2\sigma^2}\right) \frac{\diff x}{\sigma\sqrt{2\pi}} =  \frac{\nu}{\sqrt{\nu^2+\sigma^2}} \cdot  \exp\left(\frac{-(\xi-\mu)^2}{2(\nu^2+\sigma^2)}\right)\cdot  \\
&\left[\frac{\sigma^2\xi + \nu^2\mu}{\nu^2+\sigma^2} \cdot \frac{1}{2}\erfun{\frac{\nu^2(x-\mu)+\sigma^2(x-\xi)}{\nu\sigma\sqrt{2(\nu^2+\sigma^2)}}} - \frac{\nu\sigma}{\sqrt{2\pi}\sqrt{\nu^2+\sigma^2}}\cdot \exp\left(-\left(\frac{\nu^2(x-\mu)+\sigma^2(x-\xi)}{\nu\sigma\sqrt{2(\nu^2+\sigma^2)}}\right)^2\right)\right]
\, .
\end{align*}
In particular, 
\[
\int_{-\infty}^{+\infty} x\cdot \exp\left(\frac{-(x-\xi)^2}{2\nu^2}+\frac{-(x-\mu)^2}{2\sigma^2}\right) \frac{\diff x}{\sigma\sqrt{2\pi}} = \frac{\sigma^2\xi + \nu^2\mu}{\nu^2+\sigma^2} \cdot \frac{\nu}{\sqrt{\nu^2+\sigma^2}} \cdot  \exp\left(\frac{-(\xi-\mu)^2}{2(\nu^2+\sigma^2)}\right)
\, .
\]
\end{lemma}

\begin{proof}
For any $a,b,c$ with $a >0$, it holds that
\[
\int x\cdot \exps{-ax^2+bx+c} \diff x =  \frac{\sqrt{\pi}b}{4a^{3/2}}\exps{b^2/(4a)+c}\erfun{\frac{2ax-b}{2\sqrt{a}}} - \frac{1}{2a}\exps{-ax^2+bx+c}
\, .
\]
\end{proof}

Finally we want to mention the following result. 

\begin{lemma}[Gaussian integral, $2$nd order]
\label{lemma:gaussian-second-order}
Let $\xi,\mu$ be real numbers, and $\nu,\sigma$ be positive real numbers. 
Then, it holds that
\[
\int_{-\infty}^{+\infty} x^2 \cdot \exp\left(\frac{-(x-\xi)^2}{2\nu^2} + \frac{-(x-\mu)^2}{2\sigma^2}\right) \frac{\diff x}{\sigma\sqrt{2\pi}} = 
\frac{(\sigma^2\xi + \nu^2\mu)^2 + \nu^2\sigma^2(\nu^2+\sigma^2)}{(\nu^2+\sigma^2)^2} \cdot \frac{\nu}{\sqrt{\nu^2+\sigma^2}} \cdot \exp\left(\frac{-(\xi-\mu)^2}{2(\nu^2+\sigma^2)}\right)
\, .
\]
\end{lemma}

\begin{remark}
As a consequence of Lemma~\ref{lemma:gaussian-second-order}, it would be possible to further our analysis by adding second degree terms to~$f$. 
Indeed, quantities depending on $\norm{x_i-\xi}$, which would have to be computed to extend the proofs of Lemmas~\ref{lemma:computation-expected-response-vector} and~\ref{lemma:concentration-response}, can be computed with this lemma. 
For instance, one can show that
\[
\expec{\pi_i \norm{x_i-\xi}^2} = \cst \cdot \left[\frac{\nu^4}{(\nu^2+\sigma^2)^2}\norm{\xi-\mu}^2 + \frac{\nu^2\sigma^2\Dim}{\nu^2+\sigma^2}\right]
\, .
\]
%
% TODO: this got cut somehow
%However, we feel like this distract the paper from the main message, since the  
% TODO: or maybe give the value and say bounded by M_2
\end{remark}

\begin{proof}
We use the fact that 
\[
\int x^2 \cdot \exps{-ax^2+bx+c} \diff x = \frac{\sqrt{\pi}(2a+b^2)}{8a^{5/2}} \exps{\frac{b^2}{4a}+c}\cdot \erfun{\frac{2ax-b}{2\sqrt{a}}} - \frac{ax+b}{4a^2} \cdot \exps{-ax^2+bx+c}
\, .
\]
\end{proof}

\subsection{Concentration results}

In this section we collect some concentration results used throughout our proofs. 
Note that we use rather use the two-sided version of these results. 

\begin{theorem}[Hoeffding's inequality]
\label{th:hoeffding}
Let $X_1,\ldots,X_n$ be independent random variables such that $X_i$ takes its values in $[a_i,b_i]$ almost surely for all $i\leq n$. 
Then for every $t>0$,
\[
\proba{\frac{1}{n}\sum_{i=1}^n (X_i-\expec{X_i}) \geq t} \leq \exp\left(\frac{-2t^2n^2}{\sum_{i=1}^n (b_i-a_i)^2}\right)
\, .
\]
\end{theorem}

\begin{proof}
This is Theorem~2.8 in \citet{Bou_Lug_Mas:2013} in our notation.
\end{proof}

\begin{theorem}[Hoeffding's inequality for sub-Gaussian random variables]
\label{th:hoeffding-subgaussian}
Let $X_1,\ldots,X_n$ be independent random variables such that $X_i$ is sub-Gaussian with parameter $s^2>0$. 
Then, for every $t>0$,
\[
\proba{\frac{1}{n} \sum_{i=1}^n X_i - \expec{X_i} > t} \leq \exp\left(\frac{-nt^2}{2s^2}\right)
\, .
\]
\end{theorem}

\begin{proof}
This is Proposition~2.1 in \citet{Wai:2019}.
\end{proof}

%% file: main.bbl
\begin{thebibliography}{16}
\providecommand{\natexlab}[1]{#1}
\providecommand{\url}[1]{\texttt{#1}}
\expandafter\ifx\csname urlstyle\endcsname\relax
  \providecommand{\doi}[1]{doi: #1}\else
  \providecommand{\doi}{doi: \begingroup \urlstyle{rm}\Url}\fi

\bibitem[Baehrens et~al.(2010)Baehrens, Schroeter, Harmeling, Kawanabe, Hansen,
  and M{\"u}ller]{Bae_Sch_Har:2010}
D.~Baehrens, T.~Schroeter, S.~Harmeling, M.~Kawanabe, K.~Hansen, and K.-R.
  M{\"u}ller.
\newblock How to explain individual classification decisions.
\newblock \emph{Journal of Machine Learning Research}, 11\penalty0
  (6):\penalty0 1803--1831, 2010.

\bibitem[Boucheron et~al.(2013)Boucheron, Lugosi, and
  Massart]{Bou_Lug_Mas:2013}
S.~Boucheron, G.~Lugosi, and P.~Massart.
\newblock \emph{Concentration inequalities: A nonasymptotic theory of
  independence}.
\newblock {O}xford {U}niversity {P}ress, 2013.

\bibitem[Guidotti et~al.(2019)Guidotti, Monreale, Ruggieri, Turini, Giannotti,
  and Pedreschi]{Gui_Mon_Rug:2019}
R.~Guidotti, A.~Monreale, S.~Ruggieri, F.~Turini, F.~Giannotti, and
  D.~Pedreschi.
\newblock A survey of methods for explaining black box models.
\newblock \emph{ACM Computing Surveys}, 51\penalty0 (5):\penalty0 93, 2019.

\bibitem[Harrison~Jr. and Rubinfeld(1978)]{Har_Rub:1978}
D.~Harrison~Jr. and D.~L. Rubinfeld.
\newblock Hedonic housing prices and the demand for clean air.
\newblock \emph{Journal of environmental economics and management}, 5\penalty0
  (1):\penalty0 81--102, 1978.

\bibitem[Lundberg and Lee(2017)]{Lun_Lee:2017}
S.~M. Lundberg and S.-I. Lee.
\newblock A unified approach to interpreting model predictions.
\newblock In \emph{NeurIPS}, 2017.

\bibitem[O'Leary and Stewart(1996)]{OLe_Ste:1996}
D.~P. O'Leary and G.~W. Stewart.
\newblock Computing the eigenvalues and eigenvectors of arrowhead matrices.
\newblock \emph{Journal of Computational Physics}, 90:\penalty0 497--505, 1996.

\bibitem[Ren and Malik(2003)]{Ren_Mal:2003}
X.~Ren and J.~Malik.
\newblock Learning a classification model for segmentation.
\newblock In \emph{ICCV}, 2003.

\bibitem[Ribeiro et~al.(2016)Ribeiro, Singh, and Guestrin]{Rib_Sin_Gue:2016}
M.~T. Ribeiro, S.~Singh, and C.~Guestrin.
\newblock Why should {I} trust you? explaining the predictions of any
  classifier.
\newblock In \emph{SIGKDD}, 2016.

\bibitem[Shapley(1953)]{Sha:1953}
L.~S. Shapley.
\newblock A value for $n$-person games.
\newblock \emph{Contributions to the Theory of Games}, 2\penalty0
  (28):\penalty0 307--317, 1953.

\bibitem[Shrikumar et~al.(2016)Shrikumar, Greenside, Shcherbina, and
  Kundaje]{Shr_Gre_Shc:2016}
A.~Shrikumar, P.~Greenside, A.~Shcherbina, and A.~Kundaje.
\newblock Not just a black box: Learning important features through propagating
  activation differences.
\newblock \emph{arXiv preprint arXiv:1605.01713}, 2016.

\bibitem[Shrikumar et~al.(2017)Shrikumar, Greenside, and
  Kundaje]{Shr_Gre_Kun:2017}
A.~Shrikumar, P.~Greenside, and A.~Kundaje.
\newblock Learning important features through propagating activation
  differences.
\newblock In \emph{ICML}, 2017.

\bibitem[Szegedy et~al.(2015)Szegedy, Liu, Jia, Sermanet, Reed, Anguelov,
  Erhan, Vanhoucke, and Rabinovich]{Sze_Liu_Jia:2015}
C.~Szegedy, W.~Liu, Y.~Jia, P.~Sermanet, S.~Reed, D.~Anguelov, D.~Erhan,
  V.~Vanhoucke, and A.~Rabinovich.
\newblock Going deeper with convolutions.
\newblock In \emph{CVPR}, 2015.

\bibitem[Van~der Vaart(2000)]{Van:2000}
A.~W. Van~der Vaart.
\newblock \emph{Asymptotic {S}tatistics}.
\newblock Cambridge {U}niversity {P}ress, 2000.

\bibitem[Wainwright(2019)]{Wai:2019}
M.~J. Wainwright.
\newblock \emph{High-dimensional statistics: a non-asymptotic viewpoint}.
\newblock {C}ambridge {U}niversity {P}ress, 2019.

\bibitem[Weyl(1912)]{Wey:1912}
H.~Weyl.
\newblock Das asymptotische {V}erteilungsgesetz der {E}igenwerte linearer
  partieller {D}ifferentialgleichungen (mit einer {A}nwendung auf die {T}heorie
  der {H}ohlraumstrahlung).
\newblock \emph{Mathematische Annalen}, 71\penalty0 (4):\penalty0 441--479,
  1912.

\bibitem[Zeiler and Fergus(2014)]{Zei_Fer:2014}
M.~D. Zeiler and R.~Fergus.
\newblock Visualizing and understanding convolutional networks.
\newblock In \emph{ECCV}, pages 818--833, 2014.

\end{thebibliography}
